\def\year{2022}\relax
\documentclass[letterpaper]{article}
\usepackage{aaai22} 
\usepackage{times} 
\usepackage{helvet} 
\usepackage{courier} 
\usepackage[hyphens]{url} 
\usepackage{graphicx} 
\usepackage{graphicx} 
\usepackage{natbib} 
\usepackage{caption} 
\DeclareCaptionStyle{ruled}%
{labelfont=normalfont,labelsep=colon,strut=off}
\usepackage{algorithmic}
\usepackage{amsthm}
\usepackage[utf8]{inputenc} 
\usepackage[T1]{fontenc}    
\usepackage{hyperref} 
\hypersetup{
    colorlinks=true,
    citecolor=blue, 
    linkcolor=blue,
    filecolor=magenta,      
    urlcolor=blue,
linktocpage}
\usepackage{url}            
\usepackage{booktabs}       
\usepackage{amsfonts}       
\usepackage{nicefrac}       
\usepackage{microtype}      
\usepackage{amsmath}
\usepackage{amsthm}
\usepackage{wrapfig}
\usepackage{tablefootnote}
\usepackage[table]{xcolor}
\usepackage{enumitem}
\newtheorem{theorem}{Theorem}
\newtheorem*{theorem*}{Theorem}
\newtheorem{lemma}{Lemma}
\newtheorem*{lemma*}{Lemma} 
\newtheorem{remark}{Remark}
\newtheorem*{remark*}{Remark}
\newtheorem{definition}[theorem]{Definition}
\newtheorem{assumption}[theorem]{Assumption}

\newtheorem{proposition}[theorem]{Proposition} 
\usepackage[ruled,vlined]{algorithm2e}
\usepackage{pbox}

\DeclareMathOperator*{\argmin}{argmin}
\usepackage{comment}
\title{Online Markov Decision Processes with Non-oblivious Strategic Adversary}

\author {
   {{Le Cong Dinh$^{1}$, David Henry Mguni$^2$, Long Tran-Thanh$^3$, Jun Wang$^4$, Yaodong Yang$^5$}}   
} 
\affiliations{
    $^1$University of Southampton, 
    $^2$Huawei R\&D UK,  
    $^3$University of Warwick \\ 
    $^4$University College London, 
    $^5$King's College London \\
    Corresponding to: <yaodong.yang@kcl.ac.uk, lcd1u18@soton.ac.uk>
}

\usepackage{caption}
\usepackage{subcaption}
\usepackage{microtype}
\usepackage{graphicx}
\usepackage{booktabs} 
\usepackage{amsmath}  
\usepackage{amsfonts}
\usepackage[noamssymbols,upint]{newtxmath} 

\usepackage{hyperref}
\hypersetup{
citecolor=blue,
    colorlinks=true,
    linkcolor=red,
    filecolor=magenta,      
    urlcolor=blue,
linktocpage}

\usepackage{amsmath,amsfonts,bm}

\def\eqref#1{equation~\ref{#1}}

\def\1{\bm{1}}

\def\va{{\bm{a}}}

\def\vd{{\bm{d}}}
\def\ve{{\bm{e}}}

\def\vk{{\bm{k}}}
\def\vl{{\bm{l}}}

\def\vv{{\bm{v}}}

\def\vx{{\bm{x}}}

\def\vpi{{\boldsymbol{\pi}}}

\DeclareMathAlphabet{\mathsfit}{\encodingdefault}{\sfdefault}{m}{sl}
\SetMathAlphabet{\mathsfit}{bold}{\encodingdefault}{\sfdefault}{bx}{n}

\usepackage{pbox}

\usepackage{pdfpages}

\begin{document}

\maketitle







\begin{abstract}
We study a novel setting in Online Markov Decision Processes (OMDPs) where the loss function is chosen by a \emph{non-oblivious} strategic adversary who follows a no-external regret algorithm. In this setting, we first demonstrate that MDP-Expert, an existing algorithm that works well  with oblivious adversaries  can still apply and achieve a policy regret bound of $\mathcal{O}(\sqrt{T \log(L)}+\tau^2\sqrt{ T \log(|A|)})$ where $L$ is the size of adversary's pure strategy set and $|A|$ denotes the size of agent's action space. Considering real-world games where the support size of a NE is small, we further propose a new algorithm: \emph{MDP-Online Oracle Expert} (MDP-OOE), that achieves a policy regret bound of  $\mathcal{O}(\sqrt{T\log(L)}+\tau^2\sqrt{ T k \log(k)})$ where $k$ depends only on the support size of the NE. MDP-OOE leverages the key benefit of Double Oracle in game theory and thus can solve games with prohibitively large action space. Finally, to better understand the  learning dynamics of no-regret methods,  under the same setting of no-external regret adversary in OMDPs, we introduce an algorithm that achieves last-round convergence result  to a NE.  To our best knowledge, this is first work  leading to the last iteration result  in OMDPs.\footnote{Accepted at Autonomous Agents and Multi-Agent Systems (2023): \url{https://doi.org/10.1007/s10458-023-09599-5}.} 
\end{abstract}

\section{Introduction}
Reinforcement Learning (RL) \cite{sutton2018reinforcement} provides a general solution framework for optimal decision making under uncertainty, where the agent aims to minimise its cumulative loss while interacting with the environment. While RL algorithms have shown empirical and theoretical successes in stationary environments, it is an open challenge to deal with non-stationary environments in which the loss function and/or the transition dynamics change over time~\cite{laurent2011world}.  
In tackling non-stationary environments, we are interested in designing learning algorithms that can achieve no-regret guarantee \cite{even2009online,dick2014online}, 
 where the regret is defined as the difference between the accumulated total loss  and the total loss of the best fixed stationary policy in hindsight.

There are online learning algorithms that can achieve no-external regret property with changing loss function (but not changing transition dynamics), either in the full-information \cite{even2009online,dick2014online} or the  bandit \cite{neu2010online,neu2020online} settings. However, most existing solutions are established based on the key assumption that the adversary is \emph{oblivious}, meaning the changes in loss functions do not depend on the historical trajectories of the agent. This crucial assumption limits the applicability of no-regret algorithms to many RL fields, particularly multi-agent reinforcement learning (MARL) \cite{yang2020overview}. In a multi-agent system, since all agents are learning simultaneously, one agent's adaption on its strategy will make the environment \emph{non-oblivious} from others agents' perspective. Therefore, 
to find the optimal strategy for each player, one must consider the strategic reactions from others rather than regarding them as purely oblivious.  
As such, studying no-regret algorithms against a non-oblivious adversary is a pivotal step in adapting existing   online learning techniques into MARL  settings. 
Another challenge in online learning is the non-convergence dynamics in a system. When agents apply no-regret algorithms such as Multiplicative Weights Update (MWU)~\cite{freund1999adaptive} or Follow the Regularized Leader (FTRL)~\cite{shalev2011online} to play against each other, 
the system demonstrates  behaviours that are  
\emph{Poincar\'e recurrent} \cite{mertikopoulos2018cycles},  meaning the last-round convergence can never be achieved ~\cite{bailey2018multiplicative}.
Recent works~\cite{dinh2021last,daskalakis2019last} have focused on different learning dynamics in normal-form game that can lead to last-round convergence to a Nash equilibrium (NE) while maintaining the no-regret property. Yet, when it comes to OMDPs, it still remains an open challenge of how the no-regret property and the last-round convergence can be both achieved, especially with changing loss functions. 

 

In this paper, we relax the assumption of the oblivious adversary in OMDPs and study a new setting where the loss function 
is chosen by a strategic agent that follows a no-external regret algorithm. 
This setting can be used in applications within economics to model systems  and firms  \cite{filar1997applications}, for example, an  oligopoly with a dominant player, or ongoing interactions between industry players and an authority (e.g., a government that acts as an order-setting body). Under this setting, we study how the agent can achieve different goals such as no-policy regret and last-round convergence.

Our contributions are at three folds: 
\begin{itemize}
\item 	We prove that the well-known MDP-Expert (MDP-E) algorithm \cite{even2009online} can still apply by achieving a policy regret bound of $\mathcal{O}(\sqrt{T \log(L)}+\tau^2\sqrt{ T \log(|A|)})$, and the average strategies of the agents will converge to a NE of the game. 
\item  For many real-world applications where the support size of NE is small~\cite{mcmahan2003planning,dinh2021online}, we introduce an efficient no-regret algorithm, \emph{MDP-Online Oracle Expert (MDP-OOE)}, which achieves the policy regret bound of $\mathcal{O}(\tau^2\sqrt{ T k \log(k)} +\sqrt{T\log(L)})$ against non-oblivious adversary, where $k$ depends on the  support size of the NE. 
MDP-OOE inherits the key benefits of both  Double Oracle ~\cite{mcmahan2003planning} and MDP-E~\cite{even2009online}; it can solve games with large action space  while maintaining the no-regret property. %
\item  To achieve last-round convergence guarantee  
for no-external regret algorithms, we introduce the algorithm  of  \emph{Last-Round Convergence in OMDPs (LRC-OMDPs)} such that in cases where the adversary follows a no-external regret algorithm, the dynamics  will lead to the last-round convergence to a NE. To the best of our knowledge, this is first last-iteration convergence result in OMDPs.
\end{itemize}

\section{Related Work}

\begin{table*}[t!]
\centering
\caption{The scope of our contribution in this work. }
\vspace{-0pt}
    \begin{tabular}{|c|c|c|}
    \midrule
    & \pbox{190pt}{Non-oblivious adversary within a two-player game framework}& \pbox{130pt}{Oblivious adversary in Markov Decision Process} \\
    \midrule
    \pbox{120pt}{Regret \emph{w.r.t} best policy in hindsight} & \pbox{190pt}{MDP-OOE (our contribution) \\
    $\mathcal{O}(\tau^2\sqrt{ T k \log(k)} +\sqrt{T\log(L)})$} & 
    \pbox{130pt}{OMDPs: (MDP-E)~\cite{even2009online} \\
    $\text{Reg}_{T}= \mathcal{O}(\tau^2 \sqrt{T \log(|A|)})$}
    \\
    \midrule
    \pbox{120pt}{Regret \emph{w.r.t.} value of the game} &
    \pbox{160pt}{SGs:(UCSG)~\cite{wei2017online} \\
    $\text{Reg}_{T}= \tilde{\mathcal{O}}(D^3 |S|^5 |A|+D|S|\sqrt{|A|T})$}   & OMDPs  \\
    \midrule
    \end{tabular}
\label{Table on the difference of our setting}
\vspace{-0pt}
\end{table*}

The setting of OMDPs with no-external regret adversary, though novel,  shares certain aspects in common with existing literatures in online learning and stochastic games domains. Here we review each of them. 
Many researchers have considered OMDPs with an oblivious environment, where the loss function can be set arbitrarily. The performance of the algorithm is measured by external regret: the difference between the total loss and the best stationary policy in hindsight. In this setting with stationary transition dynamics,  MDP-E  \cite{even2009online}   proved that if the agent bounds the ``local'' regret in each state, then the ``global'' regret will be bounded. \citet{neu2010online, neu2013online}  
considered the same problem with the bandit reward feedback and provided no-external regret algorithms in this setting. \citet{dick2014online} 
studied a new approach for OMDPs where the problem can be transformed into an online linear optimization form, from which no-external regret algorithms can be derived. \citet{cheung2019non} proposed a no-external regret algorithm  in the case of non-stationary transition distribution, given that the variation of the loss and transition distributions do not exceed certain variation budgets. 

In non-oblivious environment, \citet{yu2009markov} provided an example demonstrating that no algorithms can guarantee sublinear external regret against a non-oblivious adversary. Thus, in   OMDPs with non-oblivious opponents (e.g., agents using adaptive algorithms), the focus is often on finding stationary points of the system rather than finding a no-external regret algorithm \cite{leslie2020best}. In this paper, we study cases where the adversary follows an adaptive no-regret algorithm, and  tackle the hardness result of non-oblivious environments in OMDPs. 

The problem of non-oblivious adversary has also been studied in the multi-armed bandit setting, a special case of OMDPs. In this setting, \citet{arora2012online} considered $m$-memory bounded adversary and provided an algorithm with policy regret bound that depends linearly on $m$, where the policy regret includes the adversary's adaptive behaviour (i.e., see Equation (\ref{regret for the agent})). Compared to their work, our paper considers strategic adversary which turns out to be $\infty$-memory bounded adversary. Thus the algorithm suggested in \citet{arora2012online} can not be applied. Recently, \citet{dinh2021last} studied the same strategic adversary in full information normal-form setting and provided an algorithm that leads to last round convergence. However, both of the above works only studied the simplified version of OMDPs, thus they do not capture the complexity of the problem. We argue that since strategic-adversary setting has many applications due to the popularity of no-regret algorithms \cite{cesa2006prediction,zinkevich2007regret,daskalakis2017training},  it is important to study no-regret methods in more practical settings such as OMDPs. 

Stochastic games (SGs)~\cite{shapley1953stochastic,deng2021complexity} 
offer a multi-player game framework where agents jointly decide the loss and the state transition. Compared to OMDPs, the main difference is that SGs allow each player to have representation of  states, actions and rewards, thus  players can learn the representations over time and find the NE of the stochastic games~\cite{wei2017online,Tian2020OnlineLI}.  
The performance in SGs is often measured by the difference between the average loss and the value of the game (i.e., the value when both players play a NE), which is a weaker notion of regret compared to the best fixed policy in hindsight in OMDPs.  
Intuitively, the player can learn the structure of the game (i.e., transition model, reward function) over time, thus on average, the player can calculate and compete with the value of the game. 
In non-episodic settings, the Upper Confidence Stochastic Game algorithm (UCSG)~\cite{wei2017online} guarantees the regret of $\text{Reg}_{T}= \tilde{\mathcal{O}}(D^3 |S|^5 |A|+D|S|\sqrt{|A|T})$ with high probability, given that the opponent's action is observable. However, to compete with the best stationary policy, knowing the game structure does not guarantee a good performance (i.e., the performance will heavily depend on the strategic behaviour of opponents). ~\citet{Tian2020OnlineLI} proved that in the SG setting, achieving no-regret with respect to best stationary policy in hindsight is statistically hard. %
Our settings can be considered as a sub-class of SGs where only the agent controls the transition model (i.e., single controller SGs), based on this,  we try to overcome the above challenge.

We summarise the difference between our setting and  OMDPs and SGs  in Table \ref{Table on the difference of our setting}. Compared to OMDPs, we relax the assumption about oblivious environment and study a non-oblivious counterpart with strategic adversary.
Compared to SGs, we relax the assumption of knowing opponent's action in non-episodic setting and our results only require observing the loss functions. Furthermore, the performance measurement is with respect to the best stationary policy in hindsight, which is proved to be statistically hard in SGs~\cite{Tian2020OnlineLI}. Intuitively, since we considers the problem of single controller SGs, it can overcome the hardness result.  \citet{guan2016regret} studied a similar setting to our paper, where only one player affects the transition kernel of the game. By viewing the game as an online linear optimisation, it can derive the minimax equilibrium of the game. There are two main challenges of the algorithm. Firstly, it requires both players to pre-calculate the minimax equilibrium of the game and fixes to this strategy during repeated game. Thus, in the situation where the adversary is an independent agent (i.e., it follows a different learning dynamic),  the proposed algorithm can not be applied. Secondly, and most importantly, the  no regret analysis is not provided for the algorithm in \citet{guan2016regret}, thus the algorithm can not be applied in an adversary environment. We fully address both challenges in this paper. 

\section{Problem Formulations \& Preliminaries}
We consider OMDPs where at each round $t \in \mathbb{N}$, an adversary can choose the loss function ${\vl}_t$ based on the agent's history $\{\pi_1,\pi_2,\dots,\pi_{t-1}\}$. 
Formally, we have OMDPs with finite state space $S$; finite action set at each state $A$; and a fixed transition model $P$. The agent's starting state, $x_1$, is distributed according to some distribution $\mu_0$ over $S$. At time $t$, given state $x_t \in S$, the agent chooses an action $a_t \in A$, then the agent moves to a new random state $x_{t+1}$ which is determined by the fixed transition model $P(x_{t+1}|x_t,a_t)$. Simultaneously, the agent receives an immediate loss ${\vl}_t(x_t,a_t)$, in which the loss function ${\vl}_t: S \times A \to R $ is bounded in $[0,1]^{|A| \times |S|}$ and chosen by the adversary from a simplex $\Delta_L:= \{{\vl} \in \mathbb{R}^{|S||A|} | {\vl}=\sum_{i=1}^L x_i {\vl}_i,\; \sum_{i=1}^L x_i=1,\; x_i \geq 0\; \forall i\}$ where $\{{\vl}_1, {\vl}_2, \dots, {\vl}_L\}$ are the  loss vectors of the adversary. We assume zero-sum game setting where the adversary receives the loss of $-{\vl}_t(x_t,a_t)$ at round $t$ and consider popular full information feedback~\cite{even2009online,dick2014online}, meaning the agent can observe the loss function ${\vl}_t$ after each round $t$.

Against strategic adversary, the formal definition of no-external regret becomes inadequate since the adversary is allowed to adapt to the agent's action. In this paper, we adopt the same approach in \citet{arora2012online} and consider policy regret. Formally, the goal of the agent is to have minimum policy regret with respect to the best fixed policy in hindsight:
\vspace{-0pt}
\begin{equation}\label{regret for the agent}
    R_T(\pi)= \mathbb{E}_{X,A}\left[\sum_{t=1}^T {\vl}_t^{\pi_t}(X_t,A_t)\right]- \mathbb{E}_{X,A}\left[ \sum_{t=1}^T {\vl}_t^{\pi}(X_t^{\pi},A_t^{\pi})\right],
\end{equation}
where ${\vl}_t^{\pi_t}$ denotes the loss function at time $t$ while the agent follows $\pi_1, \dots, \pi_T$
and ${\vl}_t^{\pi}$ is the adaptive loss function against the fixed policy $\pi$ of the agent.
We say that the agent achieves sublinear policy regret (i.e., no-policy regret property) with respect to the best fixed strategy in hindsight if $R_T(\pi)$ satisfies:
  $ \lim_{T \to \infty} \max_{\pi} \frac{R_T(\pi)}{T}=0.$

In a general non-oblivious adversary, we prove by a counter example that it is impossible to achieve an algorithm with a sublinear policy regret~\footnote{In the multi-armed bandit setting, it is also impossible to achieve sublinear policy regret against all adaptive adversaries (see Theorem 1 in \citet{arora2012online}).}.  Suppose the agent faces an adversary such that it gives a very low loss for the agent if the action in the first round of the agent is a specific action (i.e., by fixing the loss function to $\mathbf{0}$), otherwise the adversary will give a high loss (i.e., by fixing the loss function to $\mathbf{1}$). Against this type of adversary, without knowing the specific action, the agent's policy regret in Equation (\ref{regret for the agent}) will be $\mathcal{O}(T)$. Thus, in general non-oblivious adversary case, we will have a hardness result in policy regret. To resolve the hardness result, we study strategic adversary in OMDPs. 
\begin{assumption}[Strategic Adversary]\label{assumption: Strategic adversary}
The adversary flows a no-external regret algorithm such as for any sequence of $\pi_t$:
\begin{equation*}
    \begin{aligned}
    & \;\lim_{T \to \infty} \max_{\vl} \frac{R_T(\vl)}{T}=0,\; \text{where}\\
    & R_T({\vl})=\mathbb{E}_{X,A}\left[ \sum_{t=1}^T {\vl}(X_t,A_t)\right]- \mathbb{E}_{X,A}\left[\sum_{t=1}^T {\vl}_t^{\pi_t}(X_t,A_t)\right].
    \end{aligned}
\end{equation*}
\end{assumption}
The rationale of Assumption \ref{assumption: Strategic adversary} comes from the vanilla property of no-external algorithms: without prior information, the adversary will not do worse than the best-fixed strategy in hindsight~\cite{dinh2021last}. Thus, without the priority knowledge about the agent, the adversary will have incentive to follow a no-external regret algorithm. 
In the same way as full information feedback assumption for the agent, we assume that after each round $t$, the adversary observes the agent' stationary policy distribution $\vd_{\pi_t}$.


For every policy $\pi$, we define $P(\pi)$ the state transition matrix induced by $\pi$ such that $P(\pi)_{s, s'}=\sum_{a \in A} \pi(a|s)P^a_{s,s'}$. 
We assume through the paper that we have the mixing time assumption, which is a common assumption in OMDPs~\cite{even2009online,dick2014online,neu2013online}:
\begin{assumption}[Mixing time]\label{Mixing time}
There exists a constant $\tau >0$ such that for all distributions $\vd$ and $\vd'$ over the state space, any policy $\pi$,
\[\big\|\vd P(\pi)-\vd'P(\pi)\big\|_1 \leq e^{-1/\tau} \big\|\vd-\vd'\big\|_1 ,\]
where $\|\vx\|$ denotes the ${l}_1$ norm of a vector $\vx$. 
\end{assumption}
Denote $\vv_t^{\pi}(x,a)$ the probability of (state, action) pair $(x,a)$ at time step $t$ by following policy $\pi$ with initial state $x_1$.
Following Assumption \ref{Mixing time}, for any initial states, $\vv_t^{\pi}$ will converge to a stationary distribution $\vd_{\pi}$ as $t$ goes to infinity. Denote $\vd_{\Pi}$ the stationary distribution set from all agent's deterministic policies. With a slight abuse of notation, when an agent follows an algorithm $A$ with use $\pi_1, \pi_2, \dots$ at each time step, we denote $\vv_t(x,a)=\mathbb{P}\left[X_t=x, A_t=a\right],\;\; \vd_t=\vd_{\pi_t}$.
Thus, the regret in Equation (\ref{regret for the agent}) can be expressed as
\vspace{-0pt}
\begin{equation*}
    R_T(\pi)= \mathbb{E}\left[ \sum_{t=1}^T \big\langle {\vl}_t^{\pi_t}, \vv_t \big\rangle \right]-\mathbb{E}\left[ \sum_{t=1}^T \big\langle {\vl}_t^{\pi}, \vv_t^{\pi} \big\rangle \right].
\end{equation*}
Assumption \ref{Mixing time} allows us to define the average loss of policy $\pi$ in an online MDP with loss $\vl$ as
$\eta_{\vl}(\pi)= \langle \vl, \vd_{\pi} \rangle$
and the accumulated loss $Q_{\pi,\vl}(s,a)$ is defined as
\vspace{-0pt}
\[Q_{\pi,\vl}(s,a)=E\left[\sum_{t=1}^\infty \big(\vl(s_t,a_t)-\eta_{\vl}(\pi)\big)\Big|s_1=s, a_1=a,\pi \right].\]
As the dynamic between the agent and adversary is zero-sum, we can apply the minimax theorem~\cite{neumann1928theorie}:
\vspace{-0pt}
\begin{equation}\label{minimax theorem}
    \min_{\vd_{\pi} \in \Delta_{\vd_{\Pi}}} \max_{{\vl} \in \Delta_L} \langle {\vl}, \vd_{\pi} \rangle =\max_{{\vl} \in \Delta_L} \min_{\vd_{\pi} \in \Delta_{\vd_{\Pi}}}  \langle {\vl}, \vd_{\pi} \rangle =v.
\end{equation}
The saddle point $({\vl}, \vd_{\pi})$ that satisfies Equation (\ref{minimax theorem}) is the NE of the game~\cite{nash1950equilibrium} and $v$ is the called the value of the game.
Our work is based on no-external regret algorithms in normal-form game such as Multiplicative Weights Update~\cite{freund1999adaptive}, which is described as
\begin{definition}[Multiplicative Weights Update] Let $\vk_1, \vk_2, ...$ be a sequence of feedback received by the agent. The agent is said to follow the MWU if strategy $\tilde{\vpi}_{t+1}$ is updated as follows 
\vspace{-0pt}
\begin{equation}\label{eq: MWU update}
    \begin{aligned}
    \tilde{\vpi}_{t+1}(i)=\tilde{\vpi}_t(i)  \frac{\exp(-\mu_t \vk_t(\va^i))}{\sum_{i=1}^n \tilde{\vpi}_t(i)\exp(-\mu_t \vk_t(\va^i))}, \forall i \in [n],
    \end{aligned}
\end{equation}
 where $\mu_t >0$ is a parameter, $n$ is the number of pure strategies (i.e.,  experts) and $\tilde{\vpi}_0=[1/n,\dots,1/n]$.
\end{definition}
We also consider $\epsilon$-Nash equilibrium of the game:
\begin{definition}[$\epsilon$-Nash equilibrium]
Assume $\epsilon >0$. We call a point $({\vl}, \vd_{\pi}) \in \Delta_L \times \Delta_{\vd_{\Pi}}$ $\epsilon$-NE if:
\begin{equation*}
    \max_{{\vl} \in \Delta_L} \langle {\vl}, \vd_{\pi} \rangle -\epsilon \leq \langle {\vl}, \vd_{\pi} \rangle \leq \min_{\vd_{\pi} \in \Delta_{\vd_{\Pi}}}  \langle {\vl}, \vd_{\pi} \rangle +\epsilon.
\end{equation*}
\end{definition}
Under the setting of OMDPs against strategic adversary who aims to minimise the external regret  (i.e., Assumption \ref{assumption: Strategic adversary}), we study several properties that the agent can achieve such as no-policy regret and last round convergence.
\section{MDP-Expert against Strategic Adversary}\label{section no-regret algorithm against non-oblivious opponent}
\begin{algorithm}
\caption{MDP-Expert (MDP-E)} 
\label{MDP-expert}
\begin{algorithmic}[1]
\STATE {\bfseries Input:} Expert algorithm $B_s$ (i.e., MWU) for each state
\FOR{\text{$t=1$ to $\infty$}}
\STATE Using algorithm $B_s$ with set of expert $A$ and the feedback $Q_{\pi_t,{\vl}_t}(s,.)$ for each state $s$
\STATE Output $\pi_{t+1}$ and observe $\vl_{t+1}$
\ENDFOR
\end{algorithmic}
\vspace{-0pt}
\end{algorithm}
When the agent plays against a non-oblivious opponent, one challenge is that the best fixed policy $\pi$ is not based on the current loss sequence $[{\vl}_1, {\vl}_2,\dots]$ of the agent but a different loss sequence $[{\vl}_1^\pi, {\vl}_2^\pi \dots]$ induced by the policy $\pi$. Thus, to measure the regret in the case of a non-oblivious opponent, we need information on how the opponent will play against a fixed policy $\pi$. Under Assumption \ref{assumption: Strategic adversary}, we prove that existing MDP-E~\cite{even2009online} method, which is designed for oblivious adversary, will have no- policy regret property against the non-oblivious strategic adversary in our setting. Intuitively, MDP-E maintains a no-external regret algorithm (i.e., MWU) in each state to bound the local regret, thus the global regret can be bounded accordingly. The pseudocode of MDP-E is given in Algorithm~\ref{MDP-expert}. 

The following lemma links the relationship between the external-regret of the adversary and the regret with respect to the policy stationary distribution:
\begin{lemma} \label{regret of adversary with respect to stationary distribution}
Under MDP-E played by the agent, the external-regret of the adversary in Assumption \ref{assumption: Strategic adversary} can be expressed as:
\vspace{-0pt}
\begin{equation*}
\begin{aligned}
    R_T({\vl})&=\mathbb{E}_{X,A}\left[ \sum_{t=1}^T {\vl}(X_t,A_t)\right]- \mathbb{E}_{X,A}\left[\sum_{t=1}^T {\vl}_t^{\pi_t}(X_t,A_t)\right] \\
    &=\sum_{t=1}^T  \langle \vl, \vd_{\pi_t} \rangle - \sum_{t=1}^T \langle \vl_t, \vd_{\pi_t} \rangle +\mathcal{O}\big(\tau^2 \sqrt{T \log(|A|)}\big).
\end{aligned}
\end{equation*}
(We provide the full proof in the Appendix {\color{blue}{A.1}}.)
\vspace{-0pt}
\end{lemma}
Based on Lemma \ref{regret of adversary with respect to stationary distribution}, we can tell that the sublinear regret will hold if and only if the adversary maintains a sublinear regret with respect to the agent's policy stationary distribution. As we assume that after each time $t$, the adversary can observe the stationary distribution $\vd_{\pi_t}$, then by applying standard no-regret algorithm for online linear optimization against the feedback $\vd_{\pi_t}$ (i.e., MWU), the adversary can guarantee a good performance for himself. Thus, the Assumption \ref{assumption: Strategic adversary} for the adversary is justifiable. 

In the rest of the paper, with out loss of generality, we will study the case where the external-regret of the adversary with respect to agent's policy stationary distribution has the following bound (i.e., the adversary follows optimal no-external regret algorithms such as MWU, FTRL with respect to policy stationary distribution of the agent~\footnote{If the adversary does not follow the optimal bound (i.e., irational), then regret bound of the agent will change accordingly.}):
\[\max_{\vl \in \Delta_L}\left(\sum_{t=1}^T  \langle \vl, \vd_{\pi_t} \rangle - \sum_{t=1}^T \langle \vl_t, \vd_{\pi_t} \rangle \right)=\sqrt{\frac{T \log(L)}{2}}.\]
The next lemma provides a lower bound for performance of a fixed policy of the agent against a strategic adversary.
\begin{lemma}\label{lemma when agent uses a fixed strategy}
Suppose the agent follows a fixed stationary strategy $\pi$, then the adversary will converge to the best response to the fixed stationary strategy and
\vspace{-0pt}
\[\sum_{t=1}^T \langle {\vl}_t^{\pi}, \vd_\pi \rangle \geq T v - \sqrt{\frac{T \log(L)}{2}}.\]
\vspace{-0pt}
\end{lemma}
\vspace{-0pt}
The full proof is provided in the Appendix {\color{blue}{A.2}}. From Lemma \ref{lemma when agent uses a fixed strategy}, we can prove the following theorem:
\begin{theorem}\label{theorem on bound of stationary distribution}
Suppose the agent follows MDP-E Algorithm \ref{MDP-expert}, then the regret with respect to the stationary distribution will be bounded by
\[\sum_{t=1}^T \big\langle {\vl}_t^{\pi_t}, \vd_{\pi_t} \big\rangle- \sum_{t=1}^T \big\langle {\vl}_t^{\pi}, \vd_{\pi} \big\rangle \leq \sqrt{\frac{T \log(L)}{2}}+3\tau \sqrt{\frac{T \log(|A|)}{2}}.\]
\end{theorem}
\begin{proof}
From Lemma \ref{lemma when agent uses a fixed strategy}, it is sufficient 
to show that
\begin{equation*}
   \sum_{t=1}^T \big\langle {\vl}_t^{\pi_t}, \vd_{\pi_t} \big\rangle \leq T v + 3 \tau \sqrt{\frac{T \log(|A|)}{2}}.
\end{equation*}
Since the agent uses a no-regret algorithm with respect to the stationary distribution(e.g., MDP-E), following the same argument in Theorem 5.3 in \cite{even2009online} we have:
\vspace{-0pt}
\begin{equation*}
    \sum_{t=1}^T \big\langle {\vl}_t^{\pi_t}, \vd_{\pi_t} \big\rangle \leq T \min_{\vd_{\pi}}\big\langle \hat{{\vl}}, \vd_{\pi} \big\rangle + 3 \tau \sqrt{\frac{T \log(|A|)}{2}},
\end{equation*}
where $\hat{{\vl}}=\frac{1}{T} \sum_{t=1}^T {\vl}_t^{\pi_t}$. From the minimax equilibrium, we also have
\[\min_{\vd_{\pi}}\langle \hat{{\vl}}, \vd_{\pi} \rangle \leq \max_{{\vl} \in\Delta_L} \min_{\vd_{\pi} \in \vd_{\Pi}}  \langle {\vl}, \vd_{\pi} \rangle =v.\]
Thus, the proof is complete.
\end{proof}
Now, we can make the link between the stationary regret and the regret of the agent in Equation (\ref{regret for the agent}).  

\begin{theorem}\label{Theorem of bounding regret}
Suppose the agent follows MDP-E Algorithm \ref{MDP-expert}, then the agent's regret in Equation (\ref{regret for the agent}) will be bounded by
\begin{equation*}
    \begin{aligned}
   R_T(\pi)=\mathcal{O}(\sqrt{T \log(L)}+ \tau^2\sqrt{ T \log(|A|)}).
    \end{aligned}
\end{equation*}
(We provide the full proof in the Appendix {\color{blue}{A.3}}.)
\end{theorem}

We note that Theorem \ref{Theorem of bounding regret} will hold true for a larger set of adversary outside Assumption \ref{assumption: Strategic adversary} (e.g., FP~\cite{brown1951iterative}) 
satisfying the following property: for every fixed policy of the agent, the adversary's policy converges to the best response with respect to this fixed policy. With this property, we can bound the performance of agent's fixed policy in Lemma \ref{lemma when agent uses a fixed strategy} and thus derive the regret bound of the algorithm. Note that the regret bound in Theorem \ref{Theorem of bounding regret} will depend on the rate of convergence to best response against agent's fixed policy.

As we have shown in previous theorems, the dynamic of playing no-regret algorithm in OMDPs against strategic adversary can be interpreted as a two-player zero-sum game setting with the corresponding stationary distribution. 
From the classical saddle point theorem  \cite{freund1999adaptive}, if both player follows a no-regret algorithm then the average strategies will converge to the saddle point (i.e., a NE). 

\begin{theorem}\label{theorem: average convergence to NE}
Suppose the agent follows MDP-E, then the average strategies of both the agent and the adversary will converge to the $\epsilon_t$-Nash equilibrium of the game with:
\[\epsilon_t=\sqrt{\frac{\log(L)}{2T}}+3 \tau \sqrt{\frac{\log(|A|)}{2T}}\]
(We provide the full proof in the Appendix {\color{blue}{A.4}}.)
\vspace{-0pt}
\end{theorem}

With the sublinear convergence rate to an NE, the dynamic between MDP-E and no-regret adversary (i.e., MWU) provides an efficient method to solve the single-controller SGs.

\section{MDP-Online Oracle Expert Algorithm}
\label{sec:mdoop}
As shown in the previous section, we can bound the regret in Equation (\ref{regret for the agent}) by bounding the regret with respect to the stationary distribution. In MDP-E, the regret bound (i.e., $\mathcal{O}\big(\sqrt{T \log(L)}+ \tau^2\sqrt{ T \log(|A|)}\big)$) depends on the size of pure strategy set (i.e., $|A|$) thus it becomes less efficient when the agent has a prohibitively large pure strategy set. 

Interestingly, recent paper by \citet{dinh2021online} suggested that on normal-form games, it is possible to achieve a better regret bound  where it only depends on the support size of NE rather than $|A|$. Unfortunately, extending this finding for OMDPs is highly non-trivial. 
The method in \citet{dinh2021online} is designed for normal-form games only; in the worst scenario, its  regret bound will depend on the size of pure strategy set, which is huge under our settings (i.e., $|A|^{|S|}$).

In this section, we provide a no-policy regret algorithm: MDP-Online Oracle Expert (MDP-OOE). It achieves the regret bound that only depends on the size of NE support rather than the size of the game. 
We start from presenting the small NE support size  assumption. 
\begin{assumption}[Small Support Size of NE]\label{small support size assumption}
Let $(\vd_{\vpi^*}, \vl^*)$ be a Nash equilibrium of the game of size ${ |A|^{|S|} \times L}$. We assume the support size of $(\vd_{\vpi^*}, \vl^*)$ is smaller than the game size:
$\max\big(|\operatorname{supp}(\vd_{\vpi^*})|, |\operatorname{supp}(\vl^*)|\big) <  \min(|A|^{|S|},L).$
\label{keyassump}
\end{assumption} 
Note that  the assumption of small support size of NE holds in many real-world  games~\cite{czarnecki2020real, dinh2021online,perez2021modelling,liu2021unifying,yang2021diverse}. In addition, 
we prove that such an assumption also holds in cases where  the loss vectors $[\vl_1, ..., \vl_L]$ is sampled from a continuous distribution and the size of the loss vector set $L$ is small compared to the agent's pure strategy set, that is, $|A|^{|S|} \gg L$, thus further justifying the generality of this assumption. 
\begin{lemma}\label{lemma: small support of NE}
Suppose that all loss functions are sampled from a continuous distribution and the size of the loss function set is small compared to the agent's pure strategy set (i.e., $|A|^{|S|} \gg L$). Let $(\vd_{\vpi^*}, \vl^*)$ be a Nash equilibrium of the game of size ${ |A|^{|S|} \times L}$. Then we have:
\begin{equation*}
\max\big(|\operatorname{supp}(\vd_{\vpi^*})|, |\operatorname{supp}(\vl^*)|\big) \leq L.
\end{equation*}
(We provide the full proof in the Appendix {\color{blue}{B.1}}.)
\end{lemma}
\vspace{-0pt}
Since the pure strategy set of the adversary $L$ is much smaller compared to pure strategy set of the agent $|A|^{|S|}$, the support size of NE will highly likely be  smaller compared to the size of agent's strategy set. Thus the agent can exploit this extra information to achieve better performance. 

We now present the MDP-Online Oracle Expert (MDP-OOE) algorithm as follow. 
MDP-OOE maintains a set of effective strategy $A^s_t$ in each state. In each iteration, the best response with respect to the average loss function will be calculated. If all the action in the best response are included in the current effective strategy set $A^s_t$ for each state, then the algorithm continues with the current set $A^s_t$ in each state. Otherwise, the algorithm updates the set of effective strategy in step $8$ and $9$ of Algorithm \ref{MDP online oracle expert}. We define the period of consecutive iterations as one \emph{time window} $T_i$ in which the set of effective strategy $A^s_t$ stays fixed, i.e., $ T_i:=\big\{t \ \big| \  |A^s_t|=i\big\} $. Intuitively, since both the agent and the adversary use a no-regret algorithm to play, the average strategy of both players will converge to the NE of the game. Under the small NE support size assumption, the size of the agent's effective strategy set is also small compared to the whole pure strategy set (i.e., $|A|^{|S|}$). MDP-OOE ignores the pure strategies with poor average performance and only considers ones with high average performance. The regret bound with respect to the agent's stationary distribution is given as follow:
\begin{algorithm}[t!]
\caption{MDP-Online Oracle Expert}
\label{MDP online oracle expert}
\begin{algorithmic}[1]
\STATE {\bfseries Initialise:} Sets $A^1_0, \dots A^S_0$ of effective strategy set in each state
\FOR{\text{$t=1$ to $\infty$}}
\STATE $\pi_t=BR(\bar{{\vl}})$
\IF{$\pi_t(s,.) \in A^s_{t-1}$ for all $s$} 
\STATE $A^s_{t}=A^s_{t-1}$ for all $s$
\STATE Using the expert algorithm $B_s$ with effective strategy set $A^s_{t}$ and the feedback $Q_{\pi_t,{\vl}_t}(s,.)$
\ELSIF{ there exists $\pi_t(s,.) \notin A^s_{t-1}$}
\STATE $A^s_{t}=A^s_{t-1} \cup \pi_t(s,.)$ \ \   if $\pi_t(s,.) \notin A^s_{t-1}$
\STATE $A^s_{t}=A^s_{t-1} \cup a $ \ \  if $\pi_t(s,.) \in A^s_{t-1}$ where a is randomly selected from the set $A/A^s_{t-1}$.
\STATE Reset the expert algorithm $B_s$ with effective strategy set $A^s_{t}$ and the feedback $Q_{\pi_t,{\vl}_t}(s,.)$
\ENDIF
\STATE $\bar{{\vl}}=\sum_{i=\bar{T_i}}^T {\vl}_t$
\ENDFOR
\end{algorithmic}
\vspace{-0pt}
\end{algorithm}
\begin{theorem}\label{MDP-OOE regret with stationary distribution}
Suppose the learning agent uses Algorithm \ref{MDP online oracle expert}, then the regret with respect to the stationary distribution will be bounded by:
\[\sum_{t=1}^T \big\langle {\vl}_t^{\pi_t}, \vd_{\pi_t}\big\rangle-\big\langle {\vl}_t^{\pi_t}, \vd_{\pi}\big\rangle \leq 3 \tau \left( \sqrt{2 {T k \log(k)}} +\frac{k\log(k)}{8} \right),\]
\vspace{-0pt}
where $k$ is the number of time window.

(We provide the full proof in the Appendix {\color{blue}{B.2}}.)
\end{theorem}
\vspace{-0pt}
In Algorithm \ref{MDP online oracle expert}, each time the agent updates the effective strategy set $A^s_t$ at state $s$, exactly one new pure strategy is added into the effective strategy set for each state, thus the number $k$ will be at most $|A|$. Therefore, we have the regret w.r.t the stationary distribution in the worst case will be:
\begin{equation*}
    \begin{aligned}
       3 \tau \left( \sqrt{2 {T |A| \log(|A|)}} +\frac{|A|\log(|A|)}{8} \right).
    \end{aligned}
\end{equation*}
However, as shown in \citet[Figure 1]{dinh2021online},
the number of iteration in DO method (respectively the number of time window in our setting) is linearly dependent in the support size of the NE, thus with Assumption \ref{small support size assumption}, Algorithm \ref{MDP online oracle expert} will be highly efficient.
\begin{remark}\label{remark on best reponse with respect to total average strategy}
The regret bound in Theorem \ref{MDP-OOE regret with stationary distribution} will still hold in the case we consider the total average lost instead of average lost in each time window when calculating the best response in Algorithm \ref{MDP online oracle expert}. 

(We provide the full proof  in the Appendix \textcolor{blue}{B.3}.)
\end{remark}
Given the regret with respect to policy's stationary distribution in Theorem \ref{MDP-OOE regret with stationary distribution}, we can now  derive the regret bound of Algorithm \ref{MDP online oracle expert} with respect to the true performance: 
\begin{theorem}\label{regret bound for ora algorithm}
Suppose the agent uses Algorithm \ref{MDP online oracle expert} in our online MDPs setting, then the regret in Equation \ref{regret for the agent} can be bounded by:
\vspace{-0pt}
\begin{equation*}
    R_T(\pi) =\mathcal{O}(\tau^2\sqrt{ T k \log(k)} +\sqrt{T\log(L)}).
\vspace{-0pt}
\end{equation*}

(We provide the full proof in the Appendix {\color{blue}{B.4}}.)
\end{theorem}
\vspace{-0pt}
Notably, Algorithm \ref{MDP online oracle expert} will not only reduce the regret bound in the case the number of strategy set $k$ is small, it also reduces the computational hardness of computing expert algorithm when the number of experts is prohibitively large.

\textbf{MDP-Online Oracle Algorithm with $\epsilon$-best response.} In Algorithm \ref{MDP online oracle expert}, in each iteration the agent needs to calculate the exact best response to the average loss function $\bar{{\vl}}$. Since calculating the exact best response is computationally hard and maybe infeasible in many situations \cite{vinyals2019grandmaster}, an alternative way is to consider $\epsilon$-best response.  That is, in each iteration in Algorithm \ref{MDP online oracle expert}, the agent can only access to a $\epsilon$-best response to the average loss function, where $\epsilon$ is a predefined parameter. In this situation, we provide the regret analysis for Algorithm \ref{MDP online oracle expert} as follow.
\begin{theorem}\label{theorem about epsilon-best response}
Suppose the agent only accesses to $\epsilon$-best response in each iteration when following Algorithm \ref{MDP online oracle expert}. If the adversary follows a no-regret algorithm then the average strategy of the agent and the adversary will converge to $\epsilon$-Nash equilibrium. Furthermore, the algorithm has  $\epsilon$-regret. 

(We provide the full proof in the Appendix {\color{blue}{B.5}}.)
\end{theorem}
\vspace{-0pt}
Theorem \ref{theorem about epsilon-best response} implies that by following MDP-OOE, the agent can optimise the accuracy level (in terms of $\epsilon$) based on the data that it receives to obtain the convergence rate and regret bound accordingly. 

\section{Last-Round Convergence to NE in OMDPs 
}
\label{sec:lastround}

In this section, we investigate OMDPs where the agent not only aims to minimize the regret but also stabilize the strategies. This is motivated by the fact that changing strategies through repeated games may be undesirable (e.g., see \citet{dinh2021last,daskalakis2019last}). In online learning literature, minimizing regret and achieving the system's stability are often two conflict goals. That is, if all player in a system follows a no-regret algorithm (e.g., MWU, FTRL) to minimise the regret, then the dynamic of the system will become chaotic and the strategies of players will not converge in the last round  \cite{dinh2021last,mertikopoulos2018cycles}.

To achieve the goal, we start from studying the scenarios where the agent knows its NE of the game $\pi^*$. We then propose an algorithm: Last-Round Convergence in OMDPs (LRC-OMDP) that leads to last-round convergence to NE of the game in our setting. This is the first algorithm to our knowledge that achieves last-round convergence in OMDPs where only the learning agent knows the NE of the game. Notably, this goal is non-trivial to achieve. For example, if the agent keeps following the same strategy (i.e., the NE), then while the system might be stabilised (i.e., the adversary converges to the best response), yet this is still not a no-regret algorithm. Moreover, we notice  that  understanding the  learning dynamics even when the NE is known is still challenging  in the  multi-agent learning domain. 
The AWESOME ~\cite{conitzer2007awesome} and CMLeS ~\cite{chakraborty2014multiagent} algorithms make significant efforts to achieve convergence to NE under the assumption that each agent has access to a precomputed NE strategy. 
Compared to these algorithms, LRC-OMDP enjoys the key benefit that it does not require the adversary  know its NE.
Importantly, the adversary in our setting can be any types of strategic agent who observes the history and applies a no-regret algorithm to play, rather than being a restricted opponent such as a stationary opponent in AWESOME or a memory-bounded opponent in CMLeS.

\begin{algorithm}[t!]
    \caption{Last-Round Convergence in OMDPs}
    \label{Last round convergence in OMDPs}
\begin{algorithmic}[1]
    \STATE {\bfseries Input:} Current iteration $t$
    \STATE {\bfseries Output: } Strategy $\pi_t$ for the agent
    \FOR{$t=1,2,\dots,T$}
    \IF{$t=2k-1, k \in \mathbb{N}$}
    \STATE $\pi_t= \pi^*$
    \ELSIF{$t=2k, k \in \mathbb{N}$}
    \STATE $\hat{\pi}_{t}(s)= \argmin_{a \in A} Q_{\pi^*,{\vl}_t} (s, a) \; \forall s \in S$
    \STATE $\alpha_t= \frac{v-\eta_{{\vl}_{t-1}}(\hat{\pi}_t)}{\beta}$;\quad
    $\vd_{\pi_t}= (1-\alpha_t) \vd_{\pi^*}+ \alpha_t \vd_{\hat{\pi}_{t}}$
    \STATE Output $\pi_t$ via $\vd_{\pi_t}$ 
    \ENDIF
    \ENDFOR
\end{algorithmic}
\vspace{-0pt}
\end{algorithm}

The LRC-OMDP algorithm can be described as follow. At each odd round, the agent follows the NE strategy $\pi^*$ so that in the next round, the strategy of the adversary will not deviating from the current strategy. Then, at the following even round, the agent chooses a strategy such that $\vd_{\pi_t}$ is a direction towards the NE strategy of the adversary. Depending on the distance between the current strategy of the adversary and its NE (which is measured by $v-\eta_{\vl_{t-1}}(\hat{\pi}_t)$), the agent will chooses a step size $\alpha_t$ such that the strategy of adversary will approach the NE. Note here that $\beta$ is constant parameter and depends on the specific no-regret algorithm adversary follows, there is different optimal value for $\beta$. In case where the adversary follows the MWU algorithm, we can set $\beta=1$.

We first introduce the condition in which the system achieves stability through the following lemma:
\begin{lemma}\label{lemma about last-round convergence}
Let $\pi^*$ be the NE strategy of the agent. Then, ${\vl}$ is the Nash Equilibrium of the adversary if the two following conditions hold:
\begin{equation*}
    Q_{\pi^*, {\vl}} (s, \pi^*)=\argmin_{\pi \in \Pi} Q_{\pi^*, {\vl}}(s, \pi) \;\; \forall s \in S\;\;\text{and}\;\; \eta_l(\pi^*) = v.
\end{equation*}
\end{lemma}
The above lemma implies that if there is no improvement in the Q-value function for every state and the value of the current loss function equals to the value of the game, then there is last-round convergence to the NE. In situations where there is an improvement in one state, the following lemma bounds the value of a new strategy:
\begin{lemma}\label{lemma: improvement in Q value}
Assume that $\forall \pi \in \Pi$, $\vd_{\pi}(s) >0$. Then if there exist $s \in S$ such that 
\[Q_{\pi^*, {\vl_t}} (s, \pi^*) > \argmin_{\pi \in \Pi} Q_{\pi^*, {\vl_t}}(s, \pi),\]
then for $\pi_{t+1}(s)=\argmin_{a \in A} Q_{\pi^*,{\vl}_t} (s, a) \; \forall s \in S$:
\[\eta_{{\vl}_t}(\pi_{t+1}) < v.\]
\end{lemma}
\vspace{-0pt}
Based on the above lemmas, we finally reach the last-round convergence of  LRC-MDP in  Algorithm \ref{Last round convergence in OMDPs}.

\begin{theorem}\label{convergence result for omdp}
Assume that the adversary follows the MWU algorithm with non-increasing step size $\mu_t$ such that $\lim_{T \to \infty} \sum_{t=1}^T \mu_t =\infty$ and there exists $t' \in \mathbb{N}$ with $\mu_{t'} \leq \frac{1}{3}$. If the agent follows Algorithm \ref{Last round convergence in OMDPs} then there exists a Nash equilibrium ${\vl}^*$ for the adversary such that $lim_{t \to \infty} {\vl}_t = {\vl}^*$ almost everywhere and $lim_{t \to \infty} \pi_t = \pi^*$.
\end{theorem}
\vspace{-0pt}
\begin{proof}
The full proof is given the the Appendix \textcolor{blue}{C.3}. The main idea of the proof is to consider the relative entropy distance between the NE and the current strategy of the adversary and derives that $\forall k \in \mathbb{N}: \;\; 2k\geq t'$:
\begin{equation*}
\begin{aligned}
&\operatorname{RE}\left({\vl}^*\|{{\vl}}_{2k-1}\right)-\operatorname{RE}\left({\vl}^*\|{{\vl}}_{2k+1}\right) \geq   \frac{1}{2}\mu_{2k}\alpha_{2k}\big(v-\eta_{{\vl}_{2k-1}}(\hat{\pi}_{2k})\big).
\end{aligned}
\end{equation*}
\end{proof}
\vspace{-0pt}
The Algorithm \ref{Last round convergence in OMDPs}  also applies in the situations where the adversary follows different learning dynamic such as Follow the Regularized Leader or linear MWU~\cite{dinh2021last}. In these situations, Algorithm \ref{Last round convergence in OMDPs} requires to adapt the constant parameter $\beta$ so that the convergence result still holds. Since both the agent and the adversary converges to a  NE, the NE is also the best fixed strategy in hindsight. Consequently, LRC-OMDP is also a no-regret algorithm where the regret bound depends on the convergence rate to the NE.
\section{Conclusion}
In this paper, we have studied a novel setting in Online Markov Decision Processes where the loss function is chosen by a non-oblivious strategic adversary who follows a no-external regret algorithm. In this setting, we then revisited the MDP-E algorithm and provided a sublinear regret bound for it. We suggested a new algorithm of MDP-OOE that achieves  the policy regret  of  $\mathcal{O}(\sqrt{T\log(L)}+\sqrt{\tau^2 T k \log(k)})$ where the regret does not depend the size of strategy set $|A|$ but the support size of the NE $k$. Finally, in tackling non convergence property of no-regret algorithms in self-plays, we provided the LRC-OMDP algorithm for the agent that leads to the first-known result of the last-round convergence to a NE.
\bibliography{main}

\begin{thebibliography}{40}
\providecommand{\natexlab}[1]{#1}

\bibitem[{Arora, Dekel, and Tewari(2012)}]{arora2012online}
Arora, R.; Dekel, O.; and Tewari, A. 2012.
\newblock Online bandit learning against an adaptive adversary: from regret to
  policy regret.
\newblock \emph{arXiv preprint arXiv:1206.6400}.

\bibitem[{Bailey and Piliouras(2018)}]{bailey2018multiplicative}
Bailey, J.~P.; and Piliouras, G. 2018.
\newblock Multiplicative weights update in zero-sum games.
\newblock In \emph{Proceedings of the 2018 ACM Conference on Economics and
  Computation}, 321--338.

\bibitem[{Bohnenblust, Karlin, and Shapley(1950)}]{bohnenblust1950solutions}
Bohnenblust, H.; Karlin, S.; and Shapley, L. 1950.
\newblock Solutions of discrete, two-person games.
\newblock \emph{Contributions to the Theory of Games}, 1: 51--72.

\bibitem[{Brown(1951)}]{brown1951iterative}
Brown, G.~W. 1951.
\newblock Iterative solution of games by fictitious play.
\newblock \emph{Activity analysis of production and allocation}, 13(1):
  374--376.

\bibitem[{Cesa-Bianchi and Lugosi(2006)}]{cesa2006prediction}
Cesa-Bianchi, N.; and Lugosi, G. 2006.
\newblock \emph{Prediction, learning, and games}.
\newblock Cambridge university press.

\bibitem[{Chakraborty and Stone(2014)}]{chakraborty2014multiagent}
Chakraborty, D.; and Stone, P. 2014.
\newblock Multiagent learning in the presence of memory-bounded agents.
\newblock \emph{Autonomous agents and multi-agent systems}, 28(2): 182--213.

\bibitem[{Cheung, Simchi-Levi, and Zhu(2019)}]{cheung2019non}
Cheung, W.~C.; Simchi-Levi, D.; and Zhu, R. 2019.
\newblock Non-stationary reinforcement learning: The blessing of (more)
  optimism.
\newblock \emph{Available at SSRN 3397818}.

\bibitem[{Conitzer and Sandholm(2007)}]{conitzer2007awesome}
Conitzer, V.; and Sandholm, T. 2007.
\newblock AWESOME: A general multiagent learning algorithm that converges in
  self-play and learns a best response against stationary opponents.
\newblock \emph{Machine Learning}, 67(1-2): 23--43.

\bibitem[{Czarnecki et~al.(2020)Czarnecki, Gidel, Tracey, Tuyls, Omidshafiei,
  Balduzzi, and Jaderberg}]{czarnecki2020real}
Czarnecki, W.~M.; Gidel, G.; Tracey, B.; Tuyls, K.; Omidshafiei, S.; Balduzzi,
  D.; and Jaderberg, M. 2020.
\newblock Real World Games Look Like Spinning Tops.
\newblock \emph{arXiv preprint arXiv:2004.09468}.

\bibitem[{Daskalakis et~al.(2017)Daskalakis, Ilyas, Syrgkanis, and
  Zeng}]{daskalakis2017training}
Daskalakis, C.; Ilyas, A.; Syrgkanis, V.; and Zeng, H. 2017.
\newblock Training gans with optimism.
\newblock \emph{arXiv preprint arXiv:1711.00141}.

\bibitem[{Daskalakis and Panageas(2019)}]{daskalakis2019last}
Daskalakis, C.; and Panageas, I. 2019.
\newblock Last-Iterate Convergence: Zero-Sum Games and Constrained Min-Max
  Optimization.
\newblock \emph{10th Innovations in Theoretical Computer Science}.

\bibitem[{Deng et~al.(2021)Deng, Li, Mguni, Wang, and
  Yang}]{deng2021complexity}
Deng, X.; Li, Y.; Mguni, D.~H.; Wang, J.; and Yang, Y. 2021.
\newblock On the Complexity of Computing Markov Perfect Equilibrium in
  General-Sum Stochastic Games.
\newblock \emph{arXiv preprint arXiv:2109.01795}.

\bibitem[{Dick, Gyorgy, and Szepesvari(2014)}]{dick2014online}
Dick, T.; Gyorgy, A.; and Szepesvari, C. 2014.
\newblock Online learning in Markov decision processes with changing cost
  sequences.
\newblock In \emph{ICML}, 512--520.

\bibitem[{Dinh et~al.(2021{\natexlab{a}})Dinh, Nguyen, Zemhoho, and
  Tran-Thanh}]{dinh2021last}
Dinh, L.~C.; Nguyen, T.-D.; Zemhoho, A.~B.; and Tran-Thanh, L.
  2021{\natexlab{a}}.
\newblock Last Round Convergence and No-Dynamic Regret in Asymmetric Repeated
  Games.
\newblock In \emph{Algorithmic Learning Theory}, 553--577. PMLR.

\bibitem[{Dinh et~al.(2021{\natexlab{b}})Dinh, Yang, Tian, Nieves, Slumbers,
  Mguni, and Wang}]{dinh2021online}
Dinh, L.~C.; Yang, Y.; Tian, Z.; Nieves, N.~P.; Slumbers, O.; Mguni, D.~H.; and
  Wang, J. 2021{\natexlab{b}}.
\newblock Online Double Oracle.
\newblock \emph{arXiv preprint arXiv:2103.07780}.

\bibitem[{Even-Dar, Kakade, and Mansour(2009)}]{even2009online}
Even-Dar, E.; Kakade, S.~M.; and Mansour, Y. 2009.
\newblock Online Markov decision processes.
\newblock \emph{Mathematics of Operations Research}, 34(3): 726--736.

\bibitem[{Filar and Vrieze(1997)}]{filar1997applications}
Filar, J.; and Vrieze, K. 1997.
\newblock Applications and Special Classes of Stochastic Games.
\newblock In \emph{Competitive Markov Decision Processes}, 301--341. Springer.

\bibitem[{Freund and Schapire(1999)}]{freund1999adaptive}
Freund, Y.; and Schapire, R.~E. 1999.
\newblock Adaptive game playing using multiplicative weights.
\newblock \emph{Games and Economic Behavior}, 29(1-2): 79--103.

\bibitem[{Guan et~al.(2016)Guan, Raginsky, Willett, and Zois}]{guan2016regret}
Guan, P.; Raginsky, M.; Willett, R.; and Zois, D.-S. 2016.
\newblock Regret minimization algorithms for single-controller zero-sum
  stochastic games.
\newblock In \emph{2016 IEEE 55th Conference on Decision and Control (CDC)},
  7075--7080. IEEE.

\bibitem[{Laurent et~al.(2011)Laurent, Matignon, Fort-Piat
  et~al.}]{laurent2011world}
Laurent, G.~J.; Matignon, L.; Fort-Piat, L.; et~al. 2011.
\newblock The world of independent learners is not Markovian.
\newblock \emph{International Journal of Knowledge-based and Intelligent
  Engineering Systems}, 15(1): 55--64.

\bibitem[{Leslie, Perkins, and Xu(2020)}]{leslie2020best}
Leslie, D.~S.; Perkins, S.; and Xu, Z. 2020.
\newblock Best-response dynamics in zero-sum stochastic games.
\newblock \emph{Journal of Economic Theory}, 189: 105095.

\bibitem[{Liu et~al.(2021)Liu, Jia, Wen, Yang, Hu, Chen, Fan, and
  Hu}]{liu2021unifying}
Liu, X.; Jia, H.; Wen, Y.; Yang, Y.; Hu, Y.; Chen, Y.; Fan, C.; and Hu, Z.
  2021.
\newblock Unifying Behavioral and Response Diversity for Open-ended Learning in
  Zero-sum Games.
\newblock \emph{arXiv preprint arXiv:2106.04958}.

\bibitem[{McMahan, Gordon, and Blum(2003)}]{mcmahan2003planning}
McMahan, H.~B.; Gordon, G.~J.; and Blum, A. 2003.
\newblock Planning in the presence of cost functions controlled by an
  adversary.
\newblock In \emph{Proceedings of the 20th International Conference on Machine
  Learning (ICML-03)}, 536--543.

\bibitem[{Mertikopoulos, Papadimitriou, and
  Piliouras(2018)}]{mertikopoulos2018cycles}
Mertikopoulos, P.; Papadimitriou, C.; and Piliouras, G. 2018.
\newblock Cycles in adversarial regularized learning.
\newblock In \emph{Proceedings of the Twenty-Ninth Annual ACM-SIAM Symposium on
  Discrete Algorithms}, 2703--2717. SIAM.

\bibitem[{Nash et~al.(1950)}]{nash1950equilibrium}
Nash, J.~F.; et~al. 1950.
\newblock Equilibrium points in n-person games.
\newblock \emph{Proceedings of the national academy of sciences}, 36(1):
  48--49.

\bibitem[{Neu et~al.(2010)Neu, Antos, Gy{\"o}rgy, and
  Szepesv{\'a}ri}]{neu2010online}
Neu, G.; Antos, A.; Gy{\"o}rgy, A.; and Szepesv{\'a}ri, C. 2010.
\newblock Online Markov decision processes under bandit feedback.
\newblock In \emph{NeurIPS}, 1804--1812.

\bibitem[{Neu et~al.(2013)Neu, Gy{\"o}rgy, Szepesv{\'a}ri, and
  Antos}]{neu2013online}
Neu, G.; Gy{\"o}rgy, A.; Szepesv{\'a}ri, C.; and Antos, A. 2013.
\newblock Online Markov decision processes under bandit feedback.
\newblock \emph{IEEE Transactions on Automatic Control}, 59(3): 676--691.

\bibitem[{Neu and Olkhovskaya(2020)}]{neu2020online}
Neu, G.; and Olkhovskaya, J. 2020.
\newblock Online learning in MDPs with linear function approximation and bandit
  feedback.
\newblock \emph{arXiv e-prints}, arXiv--2007.

\bibitem[{Neumann(1928)}]{neumann1928theorie}
Neumann, J.~v. 1928.
\newblock Zur theorie der gesellschaftsspiele.
\newblock \emph{Mathematische annalen}, 100(1): 295--320.

\bibitem[{Perez-Nieves et~al.(2021)Perez-Nieves, Yang, Slumbers, Mguni, Wen,
  and Wang}]{perez2021modelling}
Perez-Nieves, N.; Yang, Y.; Slumbers, O.; Mguni, D.~H.; Wen, Y.; and Wang, J.
  2021.
\newblock Modelling Behavioural Diversity for Learning in Open-Ended Games.
\newblock In \emph{International Conference on Machine Learning}, 8514--8524.
  PMLR.

\bibitem[{Shalev-Shwartz et~al.(2011)}]{shalev2011online}
Shalev-Shwartz, S.; et~al. 2011.
\newblock Online learning and online convex optimization.
\newblock \emph{Foundations and trends in Machine Learning}, 4(2): 107--194.

\bibitem[{Shapley(1953)}]{shapley1953stochastic}
Shapley, L.~S. 1953.
\newblock Stochastic games.
\newblock \emph{Proceedings of the national academy of sciences}, 39(10):
  1095--1100.

\bibitem[{Sutton and Barto(2018)}]{sutton2018reinforcement}
Sutton, R.~S.; and Barto, A.~G. 2018.
\newblock \emph{Reinforcement learning: An introduction}.
\newblock MIT press.

\bibitem[{Tian et~al.(2020)Tian, Wang, Yu, and Sra}]{Tian2020OnlineLI}
Tian, Y.; Wang, Y.; Yu, T.; and Sra, S. 2020.
\newblock Online Learning in Unknown Markov Games.

\bibitem[{Vinyals et~al.(2019)Vinyals, Babuschkin, Czarnecki, Mathieu, Dudzik,
  Chung, Choi, Powell, Ewalds, Georgiev et~al.}]{vinyals2019grandmaster}
Vinyals, O.; Babuschkin, I.; Czarnecki, W.~M.; Mathieu, M.; Dudzik, A.; Chung,
  J.; Choi, D.~H.; Powell, R.; Ewalds, T.; Georgiev, P.; et~al. 2019.
\newblock Grandmaster level in StarCraft II using multi-agent reinforcement
  learning.
\newblock \emph{Nature}, 575(7782): 350--354.

\bibitem[{Wei, Hong, and Lu(2017)}]{wei2017online}
Wei, C.-Y.; Hong, Y.-T.; and Lu, C.-J. 2017.
\newblock Online reinforcement learning in stochastic games.
\newblock \emph{arXiv preprint arXiv:1712.00579}.

\bibitem[{Yang et~al.(2021)Yang, Luo, Wen, Slumbers, Graves, Bou~Ammar, Wang,
  and Taylor}]{yang2021diverse}
Yang, Y.; Luo, J.; Wen, Y.; Slumbers, O.; Graves, D.; Bou~Ammar, H.; Wang, J.;
  and Taylor, M.~E. 2021.
\newblock Diverse Auto-Curriculum is Critical for Successful Real-World
  Multiagent Learning Systems.
\newblock In \emph{Proceedings of the 20th International Conference on
  Autonomous Agents and MultiAgent Systems}, 51--56.

\bibitem[{Yang and Wang(2020)}]{yang2020overview}
Yang, Y.; and Wang, J. 2020.
\newblock An Overview of Multi-Agent Reinforcement Learning from Game
  Theoretical Perspective.
\newblock \emph{arXiv preprint arXiv:2011.00583}.

\bibitem[{Yu, Mannor, and Shimkin(2009)}]{yu2009markov}
Yu, J.~Y.; Mannor, S.; and Shimkin, N. 2009.
\newblock Markov decision processes with arbitrary reward processes.
\newblock \emph{Mathematics of Operations Research}, 34(3): 737--757.

\bibitem[{Zinkevich et~al.(2007)Zinkevich, Johanson, Bowling, and
  Piccione}]{zinkevich2007regret}
Zinkevich, M.; Johanson, M.; Bowling, M.; and Piccione, C. 2007.
\newblock Regret minimization in games with incomplete information.
\newblock \emph{Advances in neural information processing systems}, 20:
  1729--1736.

\end{thebibliography}
\clearpage
\onecolumn
\appendix
\section{Appendix}
\subsection{MDP-Expert against Strategic Adversary}
First we provide the following lemmas and proposition:
\begin{lemma*}[Lemma 3.3 in \cite{even2009online}]
For all loss function $\vl$ in $[0,1]$ and policies $\pi$, $Q_{\vl,\pi}(s,a) \leq 3\tau$.
\end{lemma*}
\begin{lemma*}[Lemma 1 from \cite{neu2013online}]{\label{neu2013 lemma 1}}
Consider a uniformly ergodic OMDPs with mixing time $\tau$ with losses ${\vl}_t \in [0,1]^\vd$. Then, for any $T > 1$ and policy $\pi$ with stationary distribution $\vd_{\pi}$, it holds that
\begin{equation*}
   \sum_{t=1}^T | \langle {\vl}_t, \vd_{\pi} -\vv_t^{\pi} \rangle | \leq 2 \tau +2 .
\end{equation*}
\end{lemma*}
This lemma guarantees that the performance of a policy's stationary distribution is similar to the actual performance of the policy in the case of a fixed policy.

In the other case of non-fixed policy, the following lemma bound the performance of policy's stationary distribution of algorithm $A$ with the actual performance:
\begin{lemma*}[Lemma 5.2 in \cite{even2009online}]
Let $\pi_1, \pi_2,\dots$ be the policies played by MDP-E algorithm $\mathcal{A}$ and let $\tilde{\vd}_{\mathcal{A},t},\;\tilde{\vd}_{\pi_t} \in [0,1]^{|S|}$ be the stationary state distribution. Then,
\[\|\tilde{\vd}_{\mathcal{A},t}-\tilde{\vd}_{\pi_t}\|_1\leq 2\tau^2 \sqrt{\frac{\log(|A|)}{t}}+2e^{-t/\tau}.\]
\end{lemma*}
From the above lemma, since the policy's stationary distribution is a combination of stationary state distribution and the policy's action in each state, it is easy to show that:
\[\|\vv_t-\vd_{\pi_t}\|_1 \leq \|\tilde{\vd}_{\mathcal{A},t}-\tilde{\vd}_{\pi_t}\|_1\leq 2\tau^2 \sqrt{\frac{\log(|A|)}{t}}+2e^{-t/\tau}. \]
\begin{proposition}\label{MWU property}
For the MWU algorithm~\cite{freund1999adaptive} with appropriate $\mu_t$, we have:
\[R_T(\pi)= \mathbb{E } \left[\sum_{t=1}^T \vl_t(\pi_t)\right]- \mathbb{E} \left[\sum_{t=1}^T \vl_t(\pi)\right] \leq M \sqrt{\frac{T \log(n)}{2}},\]
where $\| \vl_t(.)\| \leq M$. Furthermore, the strategy $\vpi_t$ does not change quickly: $\|\vpi_t-\vpi_{t+1}\| \leq \sqrt{\frac{\log(n)}{t}}.$
\end{proposition}
\begin{proof}
For a fixed $T$, if the loss function satisfies $\vl_t(.)\| \leq 1$ then by setting $\mu_t=\sqrt{\frac{8 \log(n)}{T}}$, following Theorem 2.2 in \cite{cesa2006prediction} we have:
\begin{equation}\label{MWU bound 1st equation}
    R_T(\pi)= \mathbb{E } \left[\sum_{t=1}^T \vl_t(\pi_t)\right]- \mathbb{E} \left[\sum_{t=1}^T \vl_t(\pi)\right] \leq 1 \sqrt{\frac{T \log(n)}{2}}.
\end{equation}
Thus, in the case where $\vl_t(.)\| \leq M$, by scaling up both sides by $M$ in Equation (\ref{MWU bound 1st equation}) we have the first result of the Proposition \ref{MWU property}. For the second part, follow the updating rule of MWU we have:
\begin{subequations}
    \begin{align}
        \pi_{t+1}(i)-\pi_t(i)&=\pi_t(i)\left(\frac{\exp(-\mu_t \vl_t(\va^i))}{\sum_{i=1}^n \vpi_t(i)\exp(-\mu_t \vl_t(\va^i))}-1\right) \nonumber\\
        &\approx \pi_t(i) \left(\frac{1-\mu_t\vl_t(\va^i)}{1-\mu_t\vl_t(\pi_t)}-1\right) \label{mwu bound 2nd equation}\\
        &=\mu_t \pi_t(i) \frac{\vl_t(\pi_t)-\vl_t(\va^i)}{1-\mu_t\vl_t(\pi_t)} = \mathcal{O}(\mu_t),\nonumber
    \end{align}
\end{subequations}
where we use the approximation $e^x\approx 1+x$ for small $x$ in Equation (\ref{mwu bound 2nd equation}). Thus, the difference in two consecutive strategies $\pi_t$ will be proportional to the learning rate $\mu_t$, which is set to be $\mathcal{O}\big(\sqrt{\frac{\log(n)}{t}}\big)$. Similar result can be found in Proposition 1 in \cite{even2009online}.
\end{proof}
Now, we are ready to prove the lemmas and theorems in the paper:
\subsection{Proof of  Lemma \ref{regret of adversary with respect to stationary distribution}}\label{proof of lemma: regret of adversary with respect to stationary distribution}
\begin{lemma*}
Under MDP-E played by the agent, the external-regret of the adversary in Assumption \ref{assumption: Strategic adversary} can be expressed as:
\begin{equation*}
\begin{aligned}
    R_T({\vl})&=\mathbb{E}_{X,A}\left[ \sum_{t=1}^T {\vl}(X_t,A_t)\right]- \mathbb{E}_{X,A}\left[\sum_{t=1}^T {\vl}_t^{\pi_t}(X_t,A_t)\right] \\
    &=\sum_{t=1}^T  \langle \vl, \vd_{\pi_t} \rangle - \sum_{t=1}^T \langle \vl_t, \vd_{\pi_t} \rangle +\mathcal{O}\big(\tau^2 \sqrt{T \log(|A|)}\big).
\end{aligned}
\end{equation*}
\end{lemma*}
\begin{proof}
It is sufficient to show that for any sequence of $\vl_t$
\[\mathbb{E}_{X,A}\left[\sum_{t=1}^T {\vl}_t(X_t,A_t)\right] -\sum_{t=1}^T \langle \vl_t, \vd_{\pi_t} \rangle= \mathcal{O} (\tau^2 \sqrt{T \log(|A|)}),\]
where $\vl_t$ denotes the loss vector of the adversary when the agent follows $\pi_1,\pi_2,\dots$ (i.e., the same as $\vl_t^{\pi_t}$).

Using the consequence of Lemma 5.2 in \cite{even2009online}, for any sequence of $\vl_t$ we have:
\begin{equation}
    \begin{aligned}
    &\mathbb{E}_{X,A}\left[\sum_{t=1}^T {\vl}_t(X_t,A_t)\right] -\sum_{t=1}^T \langle \vl_t, \vd_{\pi_t} \rangle\\
    &=\sum_{t=1}^T  \langle {\vl}_t, \vv_t-\vd_{\pi_t} \rangle \leq \sum_{t=1}^T | \langle {\vl}_t, \vv_t-\vd_{\pi_t} \rangle |\leq \sum_{t=1}^T \|\vv_t-\vd_{\pi_t} \|_1 \\
    &\leq \sum_{t=1}^T 2\tau^2 \sqrt{\frac{\log(|A|)}{t}}+2e^{-t/\tau} \\
    &\leq 4\tau^2 \sqrt{T\log(|A|)}+2(1+\tau)
    = \mathcal{O} \big(\tau^2 \sqrt{T \log(|A|)}\big).
    \end{aligned}
\end{equation}
The proof is complete.
\end{proof}
\subsection{Proof of Lemma \ref{lemma when agent uses a fixed strategy}}\label{proof of lemma: lemma when agent uses a fixed strategy}
\begin{lemma*}
Suppose the agent follows a fixed stationary strategy $\pi$, then the adversary will converge to the best response to the fixed stationary strategy and
\[\sum_{t=1}^T \langle {\vl}_t^{\pi}, \vd_\pi \rangle \geq T v - \sqrt{\frac{T \log(L)}{2}}.\]
\end{lemma*}
\begin{proof}
From Lemma \ref{regret of adversary with respect to stationary distribution}, if the adversary follows a no-regret algorithm to achieve good performance in Assumption \ref{assumption: Strategic adversary}, then the adversary must follow a no-regret algorithm with respect to the policy's stationary distribution. Without loss of generality, we can assume that the adversary follows the Multiplicative Weight Update with respect to the policy's stationary distribution $\vd_{\pi}$. Then follow the property of Multiplicative Weight Update in online linear problem, we have:
\[\max_{{\vl} \in L} \langle {\vl},\vd_{\pi} \rangle-\frac{1}{T} \sum_{t=1}^T \langle {\vl}_t^{\pi}, \vd_{\pi} \rangle \leq \sqrt{\frac{\log(L)}{2 T}}.\]
From the famous minimax theorem~\cite{neumann1928theorie} we also have:
\[\max_{{\vl} \in L} \langle {\vl}, \vd_\pi \rangle \geq \min_{\vd_{\pi} \in \vd_{\Pi}}\max_{{\vl} \in L}\langle {\vl}, \vd_\pi \rangle=v.\]
Thus we have:
\begin{equation}
    \begin{aligned}
    \sum_{t=1}^T \langle {\vl}_t^{\pi}, \vd_\pi \rangle &\geq T \max_{{\vl} \in L} \langle {\vl}, \vd_\pi \rangle - \sqrt{\frac{T \log(L)}{2}} \\
    &\geq T v - \sqrt{\frac{T \log(L)}{2}}.
    \end{aligned}
\end{equation}
\end{proof}
\subsection{Proof of Theorem \ref{Theorem of bounding regret}}
\begin{theorem*}
Suppose the agent follows MDP-E Algorithm \ref{MDP-expert}, then the agent's regret in Equation (\ref{regret for the agent}) will be bounded by
\begin{equation*}
    \begin{aligned}
   R_T(\pi)=\mathcal{O}(\sqrt{T \log(L)}+ \tau^2\sqrt{ T \log(|A|)}).
    \end{aligned}
\end{equation*}
\end{theorem*}
\begin{proof}
Using the consequence of Lemma 5.2 in \cite{even2009online}, for any sequence of $\vl_t$ we have:
\begin{equation}
    \begin{aligned}
    &\sum_{t=1}^T  \langle {\vl}_t, \vv_t-\vd_{\pi_t} \rangle \leq \sum_{t=1}^T | \langle {\vl}_t, \vv_t-\vd_{\pi_t} \rangle |\leq \sum_{t=1}^T \|\vv_t-\vd_{\pi_t} \|_1 \\
    &\leq \sum_{t=1}^T 2\tau^2 \sqrt{\frac{\log(|A|)}{t}}+2e^{-t/\tau} \\
    &\leq 4\tau^2 \sqrt{T\log(|A|)}+2(1+\tau)
    = \mathcal{O} \big(\tau^2 \sqrt{T \log(|A|)}\big).
    \end{aligned}
\end{equation}
Thus we have
\begin{equation}
    \sum_{t=1}^T |\langle {\vl}_t, \vv_t-\vd_{\pi_t} \rangle| \leq 2(1+\tau) +4\tau^2 \sqrt{T \log(|A|)}.
\end{equation}
Furthermore, if the agent uses a fixed policy $\pi$ then by Lemma \ref{lemma when agent uses a fixed strategy}, we have:
\[|\sum_{t=1}^T \langle {\vl}_t, \vd_{\pi} -\vv_t^{\pi} \rangle| \leq 2 \tau +2.\]
Since the agent uses MDP-E, a no-external regret algorithm, following the same argument in Theorem 4.1 in \cite{even2009online} we have:
\[\sum_{t=1}^T \langle {\vl}_t^{\pi_t}, \vd_{\pi_t} \rangle\leq T \min_{\vd_{\pi}}\langle \hat{{\vl}}, \vd_{\pi} \rangle+ 3\tau \sqrt{\frac{T\log(|A|)}{2}} \leq Tv+3\tau \sqrt{\frac{T\log(|A|)}{2}}.\]
Along with Lemma 2, we have:
\begin{equation*}
    \begin{aligned}
     \sum_{t=1}^T \langle {\vl}_t^{\pi_t}, \vd_{\pi_t} \rangle-\sum_{t=1}^T \langle {\vl}_t^{\pi}, \vd_{\pi} \rangle \leq \left(Tv+3\tau \sqrt{\frac{T\log(|A|)}{2}}\right)-\left(T v - \sqrt{\frac{T \log(L)}{2}}\right)\\
     =3\tau \sqrt{\frac{T\log(|A|)}{2}}+\sqrt{\frac{T \log(L)}{2}}.
    \end{aligned}
\end{equation*}

Using the above two inequalities, we can bound the regret of the agent with respect to the regret of the policy's stationary distribution:
\begin{equation}
    \begin{aligned}
     R_T(\pi)&=\mathbb{E}_{x,a}\left[\sum_{t=1}^T {\vl}_t^{\pi_t}(x_t,a_t)\right]- \mathbb{E}_{x,a}\left[ \sum_{t=1}^T {\vl}_t^{\pi}(x_t^{\pi},a_t^{\pi})\right] \\
    &= \sum_{t=1}^T \langle {\vl}_t^{\pi_t}, \vv_t \rangle - \sum_{t=1}^T \langle {\vl}_t^{\pi}, \vv_t^{\pi} \rangle \\
    &\leq \sum_{t=1}^T \left(\langle {\vl}_t^{\pi_t}, \vd_{\pi_t} \rangle+ |\langle {\vl}_t^{\pi_t}, \vv_t-\vd_{\pi_t} \rangle| \right) - \sum_{t=1}^T \left(\langle {\vl}_t^{\pi}, \vd_{\pi} \rangle -|\langle {\vl}_t^{\pi},\vv_t^{\pi}- \vd_{\pi} \rangle|\right) \\
    &\leq \sum_{t=1}^T \langle {\vl}_t^{\pi_t}, \vd_{\pi_t} \rangle-\sum_{t=1}^T \langle {\vl}_t^{\pi}, \vd_{\pi} \rangle +
    2(1+\tau) +4\tau^2 \sqrt{T \log(|A|)}+ 2+ 2\tau \\
    &\leq \sqrt{\frac{T \log(L)}{2}}+3 \tau \sqrt{\frac{T \log(|A|)}{2}}+ 4(1+\tau)+4\tau^2 \sqrt{T \log(|A|)} \\
    &=\mathcal{O}(\sqrt{T \log(L)}+ \tau^2\sqrt{ T \log(|A|)}).
    \end{aligned}
\end{equation}
The proof is complete.
\end{proof}
\subsection{Proof of Theorem \ref{theorem: average convergence to NE}}
\begin{theorem*}
Suppose the agent follows MDP-E, then the average strategies of both the agent and the adversary will converge to the $\epsilon_t$-Nash equilibrium of the game with the rate:
\[\epsilon_t=\sqrt{\frac{\log(L)}{2T}}+3 \tau \sqrt{\frac{\log(|A|)}{2T}}\]
\end{theorem*}
\begin{proof}
Since the agent and the adversary use no-regret algorithms with respect to the policy's stationary distribution, we can use the property of regret bound in normal-form game to apply. Thus we have:
\begin{equation*}
    \begin{aligned}
    \max_{{\vl} \in L} \langle {\vl},\hat{\vd_{\pi}} \rangle-\frac{1}{T} \sum_{t=1}^T \langle {\vl}_t^{\pi_t}, \vd_{\pi_t} \rangle \leq \sqrt{\frac{\log(L)}{2 T}},\\
    \frac{1}{T}\sum_{t=1}^T \langle {\vl}_t^{\pi_t}, \vd_{\pi_t} \rangle-\min_{\vd_{\pi}}\langle \hat{{\vl}}, \vd_{\pi} \rangle \leq 3\tau \sqrt{\frac{\log(|A|)}{2T}},
    \end{aligned}
\end{equation*}
where $\hat{\vd_{\pi}}=\frac{1}{T} \sum_{t=1}^T \vd_{\pi_t}$ and $\hat{{\vl}}=\frac{1}{T} \sum_{t=1}^T \vl_t^{\pi_t}$.
From this, we can prove that
\begin{equation*}
    \begin{aligned}
    \langle \hat{{\vl}}, \hat{\vd_{\pi}} \rangle &\geq \min_{\vd_{\pi}}\langle \hat{{\vl}}, \vd_{\pi} \rangle \geq \frac{1}{T}\sum_{t=1}^T \langle {\vl}_t^{\pi_t}, \vd_{\pi_t} \rangle- 3\tau \sqrt{\frac{\log(|A|)}{2T}} \\
    &\geq \max_{{\vl} \in L} \langle {\vl},\hat{\vd_{\pi}} \rangle-\sqrt{\frac{\log(L)}{2 T}}-3 \tau \sqrt{\frac{\log(|A|)}{2T}},
    \end{aligned}
\end{equation*}
and,
\begin{equation*}
    \begin{aligned}
    \langle \hat{{\vl}}, \hat{\vd_{\pi}} \rangle &\leq \max_{{\vl} \in L} \langle {\vl},\hat{\vd_{\pi}} \rangle \leq \frac{1}{T}\sum_{t=1}^T \langle {\vl}_t^{\pi_t}, \vd_{\pi_t} \rangle +\sqrt{\frac{\log(L)}{2 T}} \\
    &\leq \min_{\vd_{\pi}}\langle \hat{{\vl}}, \vd_{\pi} \rangle+ 3 \tau \sqrt{\frac{\log(|A|)}{2T}} +\sqrt{\frac{\log(L)}{2 T}}.
    \end{aligned}
\end{equation*}
Thus, with $\epsilon_t=\sqrt{\frac{\log(L)}{2T}}+3 \tau \sqrt{\frac{\log(|A|)}{2T}}$, we derive
\[\max_{{\vl} \in L} \langle {\vl},\hat{\vd_{\pi}} \rangle-\epsilon_t \leq \langle \hat{{\vl}}, \hat{\vd_{\pi}} \rangle \leq \min_{\vd_{\pi}}\langle \hat{{\vl}}, \vd_{\pi} \rangle +\epsilon_t.\]
By definition, $(\hat{{\vl}}, \hat{\vd_{\pi}})$ is $\epsilon_t$-Nash equilibrium.
\end{proof}
\subsection{MDP-Online Oracle Expert Algorithm}
\subsection{Proof of Lemma \ref{lemma: small support of NE}}
\begin{lemma*}
Suppose that the loss function is sampled from a continuous distribution and the size of the loss function set is small compared to the agent's pure strategy set (i.e., $|A|^{|S|} \gg L$). Let $(\vd_{\vpi^*}, \vl^*)$ be a Nash equilibrium of the game of size ${ |A|^{|S|} \times L}$. Then we have:
\begin{equation*}
\max\big(|\operatorname{supp}(\vd_{\vpi^*})|, |\operatorname{supp}(\vl^*)|\big) \leq L.
\end{equation*}
\vspace{-15pt}
\end{lemma*}
\begin{proof}
Within the set of all zero-sum games, the set of zero-sum games with non-unique equilibrium has Lebesgue measure zero~\cite{bailey2018multiplicative}. Thus, if the loss function 's entries are sampled from a continuous distribution, then with probability one, the game has a unique NE. Following the Theorem 1 in \cite{bohnenblust1950solutions} for game with unique NE, we have:
\[|\operatorname{supp}(\vd_{\vpi^*})|=|\operatorname{supp}(\vl^*)|.\]
We also note that the support size of the NE can not exceed the size of the game:
\[|\operatorname{supp}(\vd_{\vpi^*})| \leq |A|^{|S|}; \;\;|\operatorname{supp}(\vl^*)|\leq L.\]
Thus we have:
\[\max\big(|\operatorname{supp}(\vd_{\vpi^*})|, |\operatorname{supp}(\vl^*)|\big)= |\operatorname{supp}(\vl^*)| \leq  L.\]
\end{proof}
\subsection{Proof of Theorem \ref{MDP-OOE regret with stationary distribution}}
\begin{theorem*}
Suppose the agent uses the Algorithm \ref{MDP online oracle expert}, then the regret with respect to the stationary distribution will be bounded by:
\[\sum_{t=1}^T \langle {\vl}_t^{\pi_t}, \vd_{\pi_t}\rangle-\langle {\vl}_t^{\pi_t}, \vd_{\pi}\rangle \leq 3 \tau \left( \sqrt{2 {T k \log(k)}} +\frac{k\log(k)}{8} \right),\]
where $k$ is the number of time window.
\end{theorem*}
\begin{proof}
We first have:
\begin{equation*}
    \begin{aligned}
        \mathrm{E}_{s\sim \vd_{\pi}}[Q_{\pi_t, \vl_t}(s,\pi)]&=\mathrm{E}_{s\sim \vd_{\pi},a\sim \pi}[Q_{\pi_t, \vl_t}(s,a)]\\
        &=\mathrm{E}_{s\sim\vd_{\pi},a\sim\pi}[\vl_t(s,a)-\eta_{\vl_t}(\pi_t)+\mathrm{E}_{s'\sim P_{s,a}}[Q_{\pi_t,\vl_t}(s',\pi_t)]]\\
        &=\mathrm{E}_{s\sim\vd_{\pi},a\sim\pi}[\vl_t(s,a)]-\eta_{\vl_t}(\pi_t)+\mathrm{E}_{s\sim \vd_{\pi}}[Q_{\pi_t,\vl_t}(s,\pi_t)]\\
        &=\eta_{\vl_t}(\pi)-\eta_{\vl_t}(\pi_t)+\mathrm{E}_{s\sim \vd_{\pi}}[Q_{\pi_t,\vl_t}(s,\pi_t)].
    \end{aligned}
\end{equation*}
Thus we have:
\begin{equation}\label{regret bound of algorithm2:equal1}
\langle {\vl}_t^{\pi_t}, \vd_{\pi}\rangle-\langle {\vl}_t^{\pi_t}, \vd_{\pi_t}\rangle=\sum_{s \in S} \vd_{\pi}(s)\left(Q_{\pi_t,{\vl}_t}(s,\pi)-Q_{\pi_t,{\vl}_t}(s,\pi_t)\right).
\end{equation}
Let $T_1, T_2,..., T_k$ be the time window that the $\text{BR}(\bar{{\vl}})$ does not change. Then in that time window, the best response to the current $\bar{{\vl}}$ is inside the current pure strategies set in each state. In each time window, following Equation (\ref{regret bound of algorithm2:equal1}) we have:
\begin{equation}
    \begin{aligned}
    \sum_{t=|\bar{T}_i|}^{\bar{T}_{i+1}} \langle {\vl}_t^{\pi_t}, \vd_{\pi}\rangle-\langle {\vl}_t^{\pi_t}, \vd_{\pi_t}\rangle= \sum_{s \in S} \vd_{\pi}(s) \sum_{t=|\bar{T}_i|}^{\bar{T}_{i+1}} \left(Q_{\pi_t,{\vl}_t}(s,\pi)-Q_{\pi_t,{\vl}_t}(s,\pi_t)\right).
    \end{aligned}
\end{equation}
Since during each time window, the pure strategies $A^s_t$ does not change, thus we have:
\begin{equation*}
    \begin{aligned}
    \min_{\pi \in \Pi} \sum_{t=|\bar{T}_i|}^{\bar{T}_{i+1}} \langle {\vl}_t^{\pi_t}, \vd_{\pi}\rangle = \min_{\pi \in A^s_{|\bar{T}_i|}} \sum_{t=|\bar{T}_i|}^{\bar{T}_{i+1}} \langle {\vl}_t^{\pi_t}, \vd_{\pi}\rangle .  
    \end{aligned}
\end{equation*}
Thus, in each state $s$ of a time window, the agent only needs to minimize the loss with respect to the action in $A^s_{|\bar{T}_i|}$. Put it differently, the expert algorithm in each state does not need to consider all pure action in each state, but just the current effective strategy set. For a time window $T_i$, if the agent uses a no-regret algorithm with the current effective action set and the learning rate $\mu_t=\sqrt{8\log(i)/t}$, then the regret in each state will be bounded by~\cite{cesa2006prediction}:
\begin{equation*}
    3 \tau \left( \sqrt{2 {|T_i| \log(A^s_t)}} +\frac{\log(A^s_t)}{8} \right) \leq 3 \tau \left( \sqrt{2 {|T_i| \log(i)}} +\frac{\log(i)}{8} \right).
\end{equation*}
Thus, the regret in this time interval will also be bounded by:
\begin{equation}\label{equation: regret bound in each time interval}
    \sum_{t=|\bar{T}_i|}^{\bar{T}_{i+1}} \langle {\vl}_t^{\pi_t}, \vd_{\pi_t}\rangle-\langle {\vl}_t^{\pi_t}, \vd_{\pi}\rangle\leq 3 \tau \left( \sqrt{2 {|T_i| \log(i)}} +\frac{\log(i)}{8} \right).
\end{equation}
Sum up from $i=1$ to $k$ in Inequality (\ref{equation: regret bound in each time interval}) we have:
\begin{equation}
    \begin{aligned}
    &\sum_{t=1}^T \langle {\vl}_t^{\pi_t}, \vd_{\pi_t}\rangle-\langle {\vl}_t^{\pi_t}, \vd_{\pi}\rangle =\sum_{i=1}^k \sum_{t=|\bar{T}_i|}^{\bar{T}_{i+1}} \langle {\vl}_t^{\pi_t}, \vd_{\pi_t}\rangle-\langle {\vl}_t^{\pi_t}, \vd_{\pi}\rangle \\
    &\leq \sum_{i=1}^k 3 \tau \left( \sqrt{2 {|T_i| \log(i)}} +\frac{\log(i)}{8} \right) \leq 3 \tau \left( \sqrt{2 {T k \log(k)}} +\frac{k\log(k)}{8} \right).
    \end{aligned}
\end{equation}
The proof is complete.
\end{proof}
\subsection{Proof of Remark \ref{remark on best reponse with respect to total average strategy}}
\begin{remark*}
The regret bound in Theorem \ref{MDP-OOE regret with stationary distribution} will still hold in the case we consider the total average lost instead of average lost in each time window when calculating the best response in Algorithm \ref{MDP online oracle expert}.
\end{remark*}
\begin{proof}
We prove by induction that
\[\min_{\pi \in \Pi }\sum_{t=1}^{\bar{T}_k} \langle {\vl}_t^{\pi_t}, \vd_{\pi_t}\rangle-\langle {\vl}_t^{\pi_t}, \vd_{\pi}\rangle \leq \sum_{j=1}^k \left[ \sum_{t=\bar{T}_{j-1}+1}^{\bar{T}_{j}} \langle {\vl}_t^{\pi_t}, \vd_{\pi_t}\rangle-\langle {\vl}_t^{\pi_t}, \vd_{\pi_j}\rangle \right],\]
where $\vd_{\pi_j}$ denotes the best response in the interval $[1, \bar{T}_j]$.

For $k =1$, the claim is obvious. Suppose the claim is true $k$. We then have:
\begin{subequations}
    \begin{align}
     &\min_{\pi \in \Pi }\sum_{t=1}^{\bar{T}_{k+1}} \langle {\vl}_t^{\pi_t}, \vd_{\pi_t}\rangle-\langle {\vl}_t^{\pi_t}, \vd_{\pi}\rangle =  \sum_{t=1}^{\bar{T}_{k+1}} \langle {\vl}_t^{\pi_t}, \vd_{\pi_t}\rangle-\langle {\vl}_t^{\pi_t}, \vd_{\pi_{k+1}}\rangle \nonumber\\
     &= \sum_{t=1}^{\bar{T}_{k}} \langle {\vl}_t^{\pi_t}, \vd_{\pi_t}\rangle-\langle {\vl}_t^{\pi_t}, \vd_{\pi_{k+1}}\rangle +\sum_{t=\bar{T}_{k}+1}^{\bar{T}_{k+1}} \langle {\vl}_t^{\pi_t}, \vd_{\pi_t}\rangle-\langle {\vl}_t^{\pi_t}, \vd_{\pi_{k+1}}\rangle \nonumber \\
     &\leq \min_{\pi \in \Pi }\sum_{t=1}^{\bar{T}_{k}} \langle {\vl}_t^{\pi_t}, \vd_{\pi_t}\rangle-\langle {\vl}_t^{\pi_t}, \vd_{\pi}\rangle +\sum_{t=\bar{T}_{k}+1}^{\bar{T}_{k+1}} \langle {\vl}_t^{\pi_t}, \vd_{\pi_t}\rangle-\langle {\vl}_t^{\pi_t}, \vd_{\pi_{k+1}}\rangle \nonumber \\
     &\leq \sum_{j=1}^k \left[ \sum_{t=\bar{T}_{j-1}+1}^{\bar{T}_{j}} \langle {\vl}_t^{\pi_t}, \vd_{\pi_t}\rangle-\langle {\vl}_t^{\pi_t}, \vd_{\pi_j}\rangle \right] +\sum_{t=\bar{T}_{k}+1}^{\bar{T}_{k+1}} \langle {\vl}_t^{\pi_t}, \vd_{\pi_t}\rangle-\langle {\vl}_t^{\pi_t}, \vd_{\pi_{k+1}}\rangle \label{induction step} \\
     &=\sum_{j=1}^{k+1} \left[ \sum_{t=\bar{T}_{j-1}+1}^{\bar{T}_{j}} \langle {\vl}_t^{\pi_t}, \vd_{\pi_t}\rangle-\langle {\vl}_t^{\pi_t}, \vd_{\pi_j}\rangle \right], \nonumber
    \end{align}
\end{subequations}
where the inequality (\ref{induction step}) dues to the induction assumption. Thus, for all $k$ we have:
\[\min_{\pi \in \Pi }\sum_{t=1}^{\bar{T}_k} \langle {\vl}_t^{\pi_t}, \vd_{\pi_t}\rangle-\langle {\vl}_t^{\pi_t}, \vd_{\pi}\rangle \leq \sum_{j=1}^k \left[ \sum_{t=\bar{T}_{j-1}+1}^{\bar{T}_{j}} \langle {\vl}_t^{\pi_t}, \vd_{\pi_t}\rangle-\langle {\vl}_t^{\pi_t}, \vd_{\pi_j}\rangle \right].\]
In other words, the Algorithm \ref{MDP online oracle expert} will have the same regret bound when using the best response with respect to the total average strategy of the adversary.
\end{proof}
\subsection{Proof of Theorem \ref{regret bound for ora algorithm}}
\begin{theorem*}
Suppose the agent uses the Algorithm \ref{MDP online oracle expert} in our online MDPs setting, then the regret in Equation (\ref{regret for the agent}) can be bounded by:
\begin{equation*}
    R_T(\pi) =\mathcal{O}(\tau^2\sqrt{ T k \log(k)} +\sqrt{T\log(L)}).
\end{equation*}
\end{theorem*}
\begin{proof}
First we bound the difference between the true loss and the loss with respect to the policy's stationary distribution. 
Following the Algorithm \ref{MDP online oracle expert}, at the start of each time interval $T_i$ (i.e., the time interval in which the effective strategy set does not change), the learning rate needs to restart to $\mathcal{O}(\sqrt{\log(i)/t_i})$, where $i$ denotes the number of pure strategies in the effective strategy set in the time interval $T_i$ and $t_i$ is relative position of the current round in that interval. Thus, following Lemma 5.2 in \cite{even2009online}, in each time interval $T_i$, the difference between the true loss and the loss with respect to the policy's stationary distribution will be: 
\begin{equation*}
    \begin{aligned}
        &\sum_{t=t_{i-1}+1}^{t_i}  | \langle {\vl}_t, \vv_t-\vd_{\pi_t} \rangle |\leq \sum_{t=t_{i-1}+1}^{t_i}  \|\vv_t-\vd_{\pi_t} \|_1 \\
    &\leq \sum_{t=1}^{T_i} 2\tau^2 \sqrt{\frac{\log(i)}{t}}+2e^{-t/\tau} \\
    &\leq 4\tau^2 \sqrt{T_i\log(i)}+2(1+\tau).
    \end{aligned}
\end{equation*}
From this we have:
\begin{equation*}
    \begin{aligned}
        &\sum_{t=1}^T |\langle {\vl}_t, \vv_t-\vd_{\pi_t} \rangle|=\sum_{i=1}^k \sum_{t=t_{i-1}+1}^{t_i}  | \langle {\vl}_t, \vv_t-\vd_{\pi_t} \rangle |\\
        &\leq \sum_{i=1}^k \left(4\tau^2 \sqrt{T_i\log(i)}+2(1+\tau)\right)\\
        &\leq 4\tau^2 \sqrt{Tk \log(k)}+2k(1+\tau).
    \end{aligned}
\end{equation*}
Following Lemma 1 from \cite{neu2013online}, we also have:
\[\sum_{t=1}^T| \langle {\vl}_t, \vd_{\pi} -\vv_t^{\pi} \rangle| \leq 2 \tau +2.\]
Thus the regret in Equation (\ref{regret for the agent}) can be bounded by:
\begin{equation}
\begin{aligned}
   R_T(\pi) &\leq \left( \sum_{t=1}^T \langle \vd_{\pi_t},{\vl}_t \rangle + \sum_{t=1}^T |\langle {\vl}_t, \vv_t-\vd_{\pi_t} \rangle|\right)-\left(\sum_{t=1}^T \langle {\vl}_t^{\pi}, \vd_{\pi} \rangle - \sum_{t=1}^T| \langle {\vl}_t, \vd_{\pi} -\vv_t^{\pi} \rangle|\right)\\
   &= \left(\sum_{t=1}^T \langle \vd_{\pi_t},{\vl}_t \rangle-\sum_{t=1}^T \langle {\vl}_t^{\pi}, \vd_{\pi} \rangle \right) + \sum_{t=1}^T |\langle {\vl}_t, \vv_t-\vd_{\pi_t} \rangle + \sum_{t=1}^T| \langle {\vl}_t, \vd_{\pi} -\vv_t^{\pi} \rangle| \\
    &\leq 3 \tau \left( \sqrt{2 {T k \log(k)}} +\frac{k\log(k)}{8} \right) + \frac{\sqrt{T \log(L)}}{\sqrt{2}}+ 4\tau^2 \sqrt{Tk \log(k)}+2k(1+\tau)+2\tau+2\\
    &=\mathcal{O}(\tau^2\sqrt{ T k \log(k)} +\sqrt{T\log(L)}).
\end{aligned}
\end{equation}
The proof is complete.
\end{proof}
\subsection{Proof of Theorem \ref{theorem about epsilon-best response}}
\begin{theorem*}
Suppose the agent only accesses to $\epsilon$-best response in each iteration when following Algorithm \ref{MDP online oracle expert}. If the adversary follows a no-external regret algorithm then the average strategy of the agent and the adversary will converge to $\epsilon$-Nash equilibrium. Furthermore, the algorithm has $\epsilon$-regret.
\end{theorem*}
\begin{proof}
Suppose that the player uses the Multiplicative Weights Update in Algorithm \ref{MDP online oracle expert} with $\epsilon$-best response. Let $T_1, T_2, \dots, T_k$ be the time window that the players does not add up a new strategy. Since we have a finite set of strategies $A$ then $k$ is finite. Furthermore, 
\begin{equation*}
\sum_{i=1}^k T_k=T.
\end{equation*}
In a time window $T_i$, the regret with respect to the best strategy in the set of strategy at time $T_i$ is:
\begin{equation}\label{regret in each time window: 1.1}
\begin{aligned}
    \sum_{t=\bar{T}_i}^{\bar{T}_{i+1}} \langle {\vl}_t^{\pi_t}, \vd_{\pi_t}\rangle-\min_{\pi \in A_{\bar{T}_i+1}}\sum_{t=|\bar{T}_i|}^{\bar{T}_{i+1}} \langle {\vl}_t^{\pi_t}, \vd_{\pi}\rangle\leq 3 \tau \left( \sqrt{2 {T_i \log(i)}} +\frac{\log(i)}{8} \right),
\end{aligned}
\end{equation}
where $\Bar{T}_i=\sum_{j=1}^{i-1}T_j$.
Since in the time window $T_i$, the $\epsilon$-best response strategy stays in $\Pi_{\Bar{T}_i +1}$ and therefore we have:
\[\min_{\pi \in A_{\bar{T}_i+1}} \sum_{t=|\bar{T}_i|}^{\bar{T}_{i+1}} \langle {\vl}_t^{\pi_t}, \vd_{\pi}\rangle-\min_{\pi \in \Pi} \sum_{t=|\bar{T}_i|}^{\bar{T}_{i+1}} \langle {\vl}_t^{\pi_t}, \vd_{\pi}\rangle \leq \epsilon T_i.\]
Then, from the Equation (\ref{regret in each time window: 1.1}) we have:
\begin{equation}\label{regret in each time window: 2.1}
\begin{aligned}
   \sum_{t=\bar{T}_i}^{\bar{T}_{i+1}} \langle {\vl}_t^{\pi_t}, \vd_{\pi_t}\rangle - \min_{\pi \in \Pi} \sum_{t=|\bar{T}_i|}^{\bar{T}_{i+1}} \langle {\vl}_t^{\pi_t}, \vd_{\pi}\rangle \leq 3 \tau \left( \sqrt{2 {T_i \log(i)}} +\frac{\log(i)}{8} \right)+ \epsilon T_i.
\end{aligned}
\end{equation}
Sum up the Equation (\ref{regret in each time window: 2.1}) for $i=1,\dots k$ we have:
\begin{subequations}
    \begin{align}
    \sum_{t=1}^T \langle {\vl}_t^{\pi_t}, \vd_{\pi_t}\rangle -\sum_{i=1}^k \min_{\pi \in \Pi} \sum_{t=|\bar{T}_i|}^{\bar{T}_{i+1}} \langle {\vl}_t^{\pi_t}, \vd_{\pi}\rangle
    \leq \sum_{i=1}^k 3 \tau \left( \sqrt{2 {T_i \log(i)}} +\frac{\log(i)}{8} \right)+ \epsilon T_i \nonumber \\
    \implies \sum_{t=1}^T \langle {\vl}_t^{\pi_t}, \vd_{\pi_t}\rangle -\min_{\pi \in \Pi} \sum_{i=1}^k  \sum_{t=|\bar{T}_i|}^{\bar{T}_{i+1}} \langle {\vl}_t^{\pi_t}, \vd_{\pi}\rangle \leq \epsilon T+ \sum_{i=1}^k 3 \tau \left( \sqrt{2 {T_i \log(i)}} +\frac{\log(i)}{8} \right) \label{regret 1 prove subequation 2:1} \\
    \implies \sum_{t=1}^T \langle {\vl}_t^{\pi_t}, \vd_{\pi_t}\rangle- \min_{\pi \in \Pi} \sum_{t=1}^{T} \langle {\vl}_t^{\pi_t}, \vd_{\pi}\rangle \leq  \epsilon T+ \sum_{i=1}^k 3 \tau \left( \sqrt{2 {T_i \log(i)}} +\frac{\log(i)}{8} \right)  \nonumber\\
    \implies \sum_{t=1}^T \langle {\vl}_t^{\pi_t}, \vd_{\pi_t}\rangle- \min_{\pi \in \Pi} \sum_{t=1}^{T} \langle {\vl}_t^{\pi_t}, \vd_{\pi}\rangle \leq \epsilon T +  3 \tau \left( \sqrt{2 {T k \log(k)}} +\frac{k\log(k)}{8} \right) \label{regret 1 prove subequation 2:2}.
    \end{align}
\end{subequations}
Inequality (\ref{regret 1 prove subequation 2:1}) is due to $\sum \min \leq \min \sum$. Inequality (\ref{regret 1 prove subequation 2:2}) comes from Cauchy-Schwarz inequality and Stirling' approximation. 
Using inequality (\ref{regret 1 prove subequation 2:2}), we have:
\begin{equation}\label{regret for min average}
    \min_{\pi \in \Pi} \langle \bar{{\vl}}, \vd_{\pi}  \rangle \geq \frac{1}{T} \sum_{t=1}^T \langle {\vl}_t^{\pi_t}, \vd_{\pi_t}\rangle- 3\tau \left( \sqrt{\frac{2k \log(k)}{T}} +\frac{k \log(k)}{8T} \right) -\epsilon.
\end{equation}
Since the adversary follows a no-regret algorithm, we have:
\begin{equation}\label{regret for max average}
    \begin{aligned}
    \max_{{\vl} \in \Delta_L} \sum_{t=1}^T \langle {\vl},  \vd_{\pi_t} \rangle-\sum_{t=1}^T \langle{\vl}_t^{\pi_t}, \vd_{\pi_t}\rangle \leq \sqrt{\frac{T}{2}} \sqrt{\log(L)}\\
    \implies \max_{{\vl} \in \Delta_L} \sum_{t=1}^T \langle {\vl},  \bar{\vd_{\pi}} \rangle \leq \frac{1}{T} \sum_{t=1}^T \langle{\vl}_t^{\pi_t}, \vd_{\pi_t}\rangle +\sqrt{\frac{ \log(L)}{2T}}.
    \end{aligned}
\end{equation}
Using the Inequalities (\ref{regret for min average}) and (\ref{regret for max average}) we have:
\begin{equation*}
    \begin{aligned}
    \langle \bar{{\vl}}, \bar{\vd_{\pi}} \rangle &\geq \min_{\pi \in \Pi} \langle \bar{{\vl}}, \vd_{\pi} \rangle \geq \frac{1}{T} \sum_{t=1}^T \langle {\vl}_t^{\pi_t}, \vd_{\pi_t}\rangle- 3\tau \left( \sqrt{\frac{2k \log(k)}{T}} +\frac{k \log(k)}{8T} \right) -\epsilon  \\
    &\geq \max_{{\vl} \in \Delta_L} \sum_{t=1}^T \langle {\vl},  \bar{\vd_{\pi}} \rangle- \sqrt{\frac{\log(L)}{2T}}- 3\tau \left( \sqrt{\frac{2k \log(k)}{T}} +\frac{k \log(k)}{8T} \right) -\epsilon .
    \end{aligned}
\end{equation*}
Similarly, we also have:
\begin{equation*}
    \begin{aligned}
    \langle \bar{{\vl}}, \bar{\vd_{\pi}} \rangle &\leq  \max_{{\vl} \in \Delta_L} \sum_{t=1}^T \langle {\vl},  \bar{\vd_{\pi}} \rangle \leq \frac{1}{T} \sum_{t=1}^T \langle{\vl}_t^{\pi_t}, \vd_{\pi_t}\rangle +\sqrt{\frac{ \log(L)}{2T}}\\
    &\leq \min_{\pi \in \Pi} \langle \bar{{\vl}}, \vd_{\pi}  \rangle + 3\tau \left( \sqrt{\frac{2k \log(k)}{T}} +\frac{k \log(k)}{8T} \right) +\epsilon .
    \end{aligned}
\end{equation*}
Take the limit $T \to \infty$, we then have:
\[\max_{{\vl} \in \Delta_L} \sum_{t=1}^T \langle {\vl},  \bar{\vd_{\pi}} \rangle -\epsilon \leq  \langle \bar{{\vl}}, \bar{\vd_{\pi}} \rangle  \leq \min_{\pi \in \Pi} \langle \bar{{\vl}}, \vd_{\pi}  \rangle   +\epsilon.\]
Thus $(\bar{{\vl}}, \bar{\vd_{\pi}})$ is the $\epsilon$-Nash equilibrium of the game.
\end{proof}
\subsection{Last-Round Convergence to NE in OMDPs}
\subsection{Proof of Lemma \ref{lemma about last-round convergence}}
\begin{lemma*}
Let $\pi^*$ be the NE strategy of the agent. Then, ${\vl}$ is the Nash Equilibrium of the adversary if the two following conditions hold:
\begin{equation*}
    Q_{\pi^*, {\vl}} (s, \pi^*)=\argmin_{\pi \in \Pi} Q_{\pi^*, {\vl}}(s, \pi) \;\; \forall s \in S\;\;\text{and}\;\;\eta_l(\pi^*) = v.
\end{equation*}
\end{lemma*}
\begin{proof}
Using the definition of accumulated loss function $Q$ we have
\begin{equation}
    \begin{aligned}
    &\mathbb{E}_{s \in \vd_{\pi}}[Q_{\pi^*,{\vl}}(s, \pi)]= \mathbb{E}_{s \in \vd_{\pi}, a \in \pi}[Q_{\pi^*, {\vl}}(s,a)]\\
    &= \mathbb{E}_{s \in \vd_{\pi}, a \in \pi}[{\vl}(s,a)- \eta_l(\pi^*) + \mathbb{E}_{s'\sim P_{sa}}[Q_{\pi^*, {\vl}}(s', \pi^*)]]\\
    &= \mathbb{E}_{s \in \vd_{\pi}, a \in \pi}[{\vl}(s,a)- \eta_l(\pi^*)] + \mathbb{E}_{s \in \vd_{\pi}}[Q_{\pi^*, {\vl}}(s,\pi^*)]\\
    &= \eta_l(\pi)-\eta_l(\pi^*)+ \mathbb{E}_{s \in \vd_{\pi}}[Q_{\pi^*, {\vl}}(s,\pi^*)].
    \end{aligned}
\end{equation}
Thus we have
\begin{equation}\label{equation about stationary distribution}
     \eta_l(\pi)-\eta_l(\pi^*)= \mathbb{E}_{s \in \vd_{\pi}}[Q_{\pi^*,{\vl}}(s, \pi)-Q_{\pi^*, {\vl}}(s,\pi^*).]
\end{equation}
Since we assume that 
\[\mathbb{Q}_{\pi^*, {\vl}} (s, \pi^*)=\argmin_{\pi \in \Pi} \mathbb{Q}_{\pi^*, {\vl}}(s, \pi) \;\; \forall s \in S,\]
we have
\begin{equation}
    \mathbb{Q}_{\pi^*, {\vl}}(s, \pi) \geq \mathbb{Q}_{\pi^*, {\vl}} (s, \pi^*) \;\; \forall s \in S, \pi \in \Pi.
\end{equation}
It implies that
\begin{equation}
    \mathbb{E}_{s \in \vd_{\pi}}[Q_{\pi^*,{\vl}}(s, \pi)-Q_{\pi^*, {\vl}}(s,\pi^*)] \geq 0 \;\; \forall \pi \in \Pi.
\end{equation}
Therefore we have
\begin{equation}
    \eta_l(\pi)\geq \eta_l(\pi^*) \;\; \forall \pi \in \Pi.
\end{equation}
Along with the assumption $\eta_l(\pi^*)=v$, we have the following relationship:
\begin{equation}\label{equation for Nash equilibrium requirement}
    \begin{aligned}
    \argmin_{\pi \in \Pi} \eta_l(\pi) = \eta_l(\pi^*) =v.
    \end{aligned}
\end{equation}
Now we prove that for the loss function ${\vl}$ that satisfies Equation (\ref{equation for Nash equilibrium requirement}), then ${\vl}$ is NE for the adversary. Let $(\pi^*, {\vl}^*)$ be one of the NE of the game.  Since the game we are considering is zero-sum game, $(\pi^*, {\vl}^*)$ satisfies the famous minimax theorem:
\begin{equation}
    \min_{\pi \in \Pi} \max_{{\vl}_1 \in L} \langle {\vl}_1, \vd_{\pi} \rangle= \max_{{\vl}_1 \in L}  \min_{\pi \in \Pi} \langle {\vl}_1, \vd_{\pi}\rangle = v \;\;\text{where} \;\langle {\vl}, \vd_{\pi} \rangle= \eta_l(\pi).
\end{equation}
From Equation (\ref{equation for Nash equilibrium requirement}) we have
\begin{equation} \label{inequality 1 for proof of NE for adversary}
    v = \min_{\pi \in \Pi} \langle {\vl}, \vd_{\pi} \rangle \leq  \langle {\vl}, \vd_{\pi^*} \rangle.
\end{equation}
Further, since ${\vl}^*$ is the NE of the game, then we have
\begin{equation} \label{inequality 2 for proof of NE for adversary}
    v= \langle {\vl}^*, \vd_{\pi^*} \rangle= \max_{{\vl}_1 \in L} \langle {\vl}_1, \vd_{\pi^*} \rangle \geq \langle {\vl}, \vd_{\pi^*} \rangle.
\end{equation}
From inequalities (\ref{inequality 1 for proof of NE for adversary}) and (\ref{inequality 2 for proof of NE for adversary}) we have
\begin{equation}
    v= \langle {\vl}, \vd_{\pi^*} \rangle = \min_{\pi \in \Pi} \langle {\vl}, \vd_{\pi} \rangle = \max_{{\vl}_1 \in L} \langle {\vl}_1, \vd_{\pi^*} \rangle.
\end{equation}
Thus, by definition $({\vl}, \pi^*)$ is the Nash equilibrium of the game. In other words, the loss function ${\vl}$ satisfies the above assumption is the NE of the adversary.
\end{proof}
\subsection{Proof of Lemma \ref{lemma: improvement in Q value}}
\begin{lemma*}
Assume that $\forall \pi \in \Pi$, $\vd_{\pi}(s) >0$. Then if there exists $s \in S$ such that 
\[Q_{\pi^*, {\vl_t}} (s, \pi^*) > \argmin_{\pi \in \Pi} Q_{\pi^*, {\vl_t}}(s, \pi),\]
then with a new strategy $\pi_{t+1}(s)=\argmin_{a \in A} Q_{\pi^*,{\vl}_t} (s, a) \; \forall s \in S$, we have
\[\eta_{{\vl}_t}(\pi_{t+1}) < v.\]
\end{lemma*}
\begin{proof}
From the minimax theorem, we have:
\begin{equation*}
    \eta_{{\vl}_t}(\pi^*) \leq \eta_{{\vl}^*}(\pi^*)=v \;\ \forall {\vl} \in L .
\end{equation*}
From the proof of Lemma \ref{lemma about last-round convergence} we have:
\begin{equation*}
    \eta_{{\vl}_{t}}(\pi)-\eta_{{\vl}_t}(\pi^*)= \mathbb{E}_{s \in \vd_{\pi}}[Q_{\pi^*,{\vl}_t}(s, \pi)-Q_{\pi^*, {\vl}_t}(s,\pi^*)] \; \forall \pi \in \Pi.
\end{equation*}
Since the construction of the new strategy $\pi_{t+1}$ we have:
\[\mathbb{E}_{s \in \vd_{\pi_{t+1}}}[Q_{\pi^*,{\vl}_t}(s, \pi_{t+1})-Q_{\pi^*, {\vl}_t}(s,\pi^*)] < 0,\]
thus we have:
\[\eta_{{\vl}_{t}}(\pi)< \eta_{{\vl}_t}(\pi^*) \leq 0.\]
The proof is complete.
\end{proof}
\subsection{Proof of Theorem \ref{convergence result for omdp}}\label{proof of convergence result for omdp}
\begin{theorem*}
Assume that the adversary follows the MWU algorithm with non-increasing step size $\mu_t$ such that $\lim_{T \to \infty} \sum_{t=1}^T \mu_t =\infty$ and there exists $t' \in \mathbb{N}$ with $\mu_{t'} \leq \frac{1}{3}$. If the agent follows the Algorithm \ref{Last round convergence in OMDPs} then there exists a Nash equilibrium ${\vl}^*$ for the adversary such that $lim_{t \to \infty} {\vl}_t = {\vl}^*$ almost everywhere and $lim_{t \to \infty} \pi_t = \pi^*$.
\end{theorem*}
 In order to prove the above theorem, we first need the following lemma:
\begin{lemma*} \label{Relative Entropy lemma}
\begin{equation*}
\text{RE}\left({\vl}^*\|{{\vl}}_{2k-1}\right)-\text{RE}\left({\vl}^*\|{{\vl}}_{2k+1}\right) \geq   \frac{1}{2}\mu_{2k}\alpha_{2k}(v-\eta_{{\vl}_{2k-1}}(\hat{\pi}_{2k}))\;\; \forall k \in \mathbb{N}: \;\; 2k\geq t'.
\end{equation*}
\end{lemma*}
\begin{proof}
Using the definition of relative entropy we have:
\begin{equation}
    \begin{aligned}
       &\text{RE}\left({\vl}^*\|{{\vl}}_{2k-1}\right)-\text{RE}\left({\vl}^*\|{\vl}_{2k+1}\right) \\
       &= \left(\text{RE}(\vl^*||{\vl}_{2k+1})-\text{RE}(\vl^*||{\vl}_{2k})\right)+\left(\text{RE}(\vl^*||{\vl}_{2k})-\text{RE}(\vl^*||\vl_{2k-1})\right) \nonumber\\
   &=\left(\sum_{i=1}^n \vl^*(i)\log\left(\frac{\vl^*(i)}{{\vl}_{2k+1}(i)}\right)- \sum_{i=1}^n \vl^*(i)\log\left(\frac{\vl^*(i)}{{\vl}_{2k}(i)}\right)\right) + \\
   &\quad \left( \sum_{i=1}^n \vl^*(i)\log\left(\frac{\vl^*(i)}{{\vl}_{2k}(i)}\right)- \sum_{i=1}^n \vl^*(i)\log\left(\frac{\vl^*(i)}{{\vl}_{2k-1}(i)}\right)\right)\\
   &= \left(\sum_{i=1}^n \vl^*(i)\log\left(\frac{{\vl}_{2k}(i)}{{\vl}_{2k+1}(i)}\right) \right)+ \left(\sum_{i=1}^n \vl^*(i)\log\left(\frac{{\vl}_{2k-1}(i)}{{\vl}_{2k}(i)}\right)\right).
    \end{aligned}
\end{equation}
Following the update rule of the Multiplicative Weights Update algorithm we have:
\begin{subequations}
\begin{align}
        &\text{RE}(\vl^*||{\vl}_{2k+1})-\text{RE}(\vl^*||{\vl}_{2k-1}) \nonumber\\
        &=\left(-\mu_{2k} \langle \vl^*, \vd_{\pi_{2k}} \rangle +\log(Z_{2k})\right) + \left(-\mu_{2k-1} \langle \vl^*, \vd_{\pi_{2k}} \rangle+\log(Z_{2k-1})\right) \nonumber\\
        &\leq \left(-\mu_{2k} v + \log\left(\sum_{i=1}^n {\vl}_{2k}(i)e^{\mu_{2k} \langle {\ve}_i, \vd_{\pi_{2k}} \rangle }\right)\right)+ \left(-\mu_{2k-1} v +\log(Z_{2k-1})\right) \label{MWU1aa}\\
        &=\left(-\mu_{2k} v + \log\left(\sum_{i=1}^n {\vl}_{2k-1}(i)e^{\mu_{2k-1} \langle {\ve}_i, \vd_{\pi_{2k-1}} \rangle }e^{\mu_{2k} \langle {\ve}_i, \vd_{\pi_{2k}} \rangle}\right)-\log(Z_{2k-1})\right) \nonumber\\
        & + \left(-\mu_{2k-1} v +\log(Z_{2k-1})\right),\nonumber
\end{align}
\end{subequations}
where Inequality (\ref{MWU1aa}) is due to the fact that $\langle \vl^*, \vd_{\pi} \rangle \geq v \; \forall \pi$. Thus, 
\begin{subequations}
    \begin{align}
    &\text{RE}(\vl^*||{\vl}_{2k+1})-\text{RE}(\vl^*||{\vl}_{2k-1}) \nonumber\\
       &\leq \left(-\mu_{2k} v +\log\left(\sum_{i=1}^n {\vl}_{2k-1}(i)e^{\mu_{2k-1} \langle {\ve}_i, \vd_{\pi_{2k-1}} \rangle }e^{\mu_{2k} \langle {\ve}_i, \vd_{\pi_{2k}} \rangle}\right) \right)-\mu_{2k-1} v \nonumber \\
        &\leq \left(-\mu_{2k} v +\log\left(\sum_{i=1}^n {\vl}_{2k-1}(i)e^{\mu_{2k-1} v}e^{\mu_{2k} \langle {\ve}_i, \vd_{\pi_{2k}} \rangle}\right) \right) - \mu_{2k-1} v \label{MWU1b}\\
        &= -\mu_{2k} v + \log\left(\sum_{i=1}^n {\vl}_{2k-1}(i) e^{\mu_{2k} \langle {\ve}_i, \vd_{\pi_{2k}} \rangle}\right) \nonumber,
    \end{align}
\end{subequations}
where Inequality (\ref{MWU1b}) is the result of the inequality:
\[\langle \vl, \vd_{\pi^*} \rangle \leq v \;\; \forall \vl.\]
Now, using the update rule of Algorithm \ref{Last round convergence in OMDPs}
\[\vd_{\pi_{2k}}= (1-\alpha_{2k}) \vd_{\pi^*}+ \alpha_{2k} \vd_{\hat{\pi}_{2k}},\]
we have
\begin{subequations}
    \begin{align}
    &\text{RE}(\vl^*||{\vl}_{2k+1})-\text{RE}(\vl^*||{\vl}_{2k-1}) \nonumber\\
    & \leq -\mu_{2k} v +  \log\left(\sum_{i=1}^n {\vl}_{2k-1}(i) e^{\mu_{2k} ((1-\alpha_{2k})\langle {\ve}_i, \vd_{\pi^*} \rangle + \alpha_{2k} \langle {\ve}_i, \vd_{\hat{\pi}_{2k}} \rangle)}\right) \nonumber \\
    & \leq -\mu_{2k} \alpha_{2k} v +  \log\left(\sum_{i=1}^n {\vl}_{2k-1}(i) e^{\mu_{2k} \alpha_{2k} \langle {\ve}_i, \vd_{\hat{\pi}_{2k}} \rangle}\right) \nonumber.
    \end{align}
\end{subequations}
Denote $f(\vl_{2k-1})= \langle \vl_{2k-1}, \vd_{\hat{\pi}_{2k}} \rangle$, we then have
\begin{subequations}\label{MWU3 proof 1}
\begin{align}
&\text{RE}(\vl^*||{\vl}_{2k+1})-\text{RE}(\vl^*||{\vl}_{2k-1}) \nonumber \\
& \leq -\mu_{2k} \alpha_{2k} v +\log\left(\sum_{i=1}^n {\vl}_{2k-1}(i) e^{\mu_{2k}\alpha_{2k} \langle {\ve}_i, \vd_{\hat{\pi}_{2k}} \rangle}\right) \nonumber\\
& = \mu_{2k} \alpha_{2k} (1-v) + \log\left(\sum_{i=1}^n {\vl}_{2k-1}(i) e^{-\mu_{2k}\alpha_{2k} (1-\langle {\ve}_i, \vd_{\hat{\pi}_{2k}} \rangle)}\right) \label{MWU3 proof 1a}\\
&\leq \mu_{2k} \alpha_{2k}(1-v) + \log\left(\sum_{i=1}^n {\vl}_{2k-1}(i)(1-(1-e^{-\mu_{2k}\alpha_{2k}})(1-{\langle {\ve}_i, \vd_{\hat{\pi}_{2k}} \rangle}))\right)\label{MWU3 proof 1b} \\
&=\mu_{2k} \alpha_{2k}(1-v) + \log\left(1-(1-e^{-\mu_{2k}\alpha_{2k}})(1-\langle {\vl}_{2k-1}, \vd_{\hat{\pi}_{2k}} \rangle ) \right)\nonumber\\
&\leq \mu_{2k} \alpha_{2k} (1-v) - (1-e^{-\mu_{2k}\alpha_{2k}})(1-\langle {\vl}_{2k-1}, \vd_{\hat{\pi}_{2k}} \rangle ) \label{MWU3 proof 1c}\\
&=\mu_{2k} \alpha_{2k} (1-v) -(1-e^{-\mu_{2k}\alpha_{2k}}) (1-f({\vl}_{2k-1}))\nonumber,
\end{align}
\end{subequations}
Equation (\ref{MWU3 proof 1a}) is created by adding and subtracting $\mu_{2k}\alpha_{2k}$ on the first and second terms. 

Inequalities $(\ref{MWU3 proof 1b}, \ref{MWU3 proof 1c})$ are due to
\[\beta^x \leq 1-(1-\beta)x \quad \forall \beta \geq 0 \; \vl \in [0,1] \; \text{and} \; \log(1-x) \leq -x \; \; \forall x < 1.\]
We can develop Inequality (\ref{MWU3 proof 1c}) further as 
\begin{subequations}
\begin{align}
&\text{RE}(\vl^*||{\vl}_{2k+1})-\text{RE}(\vl^*||{\vl}_{2k-1}) \nonumber\\
&\leq \mu_{2k} \alpha_{2k} (1-v) -\left(1-e^{-\mu_{2k}\alpha_{2k}}\right)(1-f({\vl}_{2k-1})) \nonumber\\
&\leq \mu_{2k} \alpha_{2k} (1-v) -\left(1-\left(1-\mu_{2k}\alpha_{2k} +\frac{1}{2}(\mu_{2k}\alpha_{2k})^2\right)\right)(1-f({\vl}_{2k-1}))\label{MWU3 proof 2a}\\
&=\mu_{2k}\alpha_{2k}(f({\vl}_{2k-1})-v) +\frac{1}{2}(\mu_{2k} \alpha_{2k})^2 (1- f({\vl}_{2k-1}))\nonumber\\
&\leq \mu_{2k}\alpha_{2k}(f({\vl}_{2k-1})-v)+\frac{1}{2}\mu_{2k}\alpha_{2k}\mu_{2k} \frac{v-f({\vl}_{2k-1})}{\beta}(1-f({\vl}_{2k-1})) \label{MWU3 proof 2c}\\
&\leq  \mu_{2k}\alpha_{2k}(f({\vl}_{2k-1})-v)+\frac{1}{2}\mu_{2k}\alpha_{2k}\ (v-f({\vl}_{2k-1})) \label{MWU3 proof 2b}\\
&=-\frac{1}{2}\mu_{2k}\alpha_{2k}(v-f({\vl}_{2k-1})) \leq 0 \nonumber .   
\end{align}
\end{subequations}
Here, Inequality ($\ref{MWU3 proof 2a} $) is due to $e^x \leq 1+x+\frac{1}{2}x^2\;\; \forall \vl \in [-\infty,0]$, Inequality (\ref{MWU3 proof 2c}) comes from the definition of $\alpha_{t}$:
\[\alpha_t = \frac{v-f({\vl}_{2k-1})}{\beta}, \; \beta \geq 1-f(\vl), \; f({\vl}_{2k-1}) \leq 1.\]
Finally, Inequality ($\ref{MWU3 proof 2b} $) comes from the choice of k at the beginning of the proof, i.e., $\mu_{2k} \leq 1$.
\end{proof}

Now we can prove Theorem \ref{convergence result for omdp}:
\begin{proof}
We focus on the regret analysis with respect to the stationary distribution $\vd_{\pi_t}$.
Let $\vl^*$ be a minimax equilibrium strategy of the adversary ($\vl^*$ may not be unique).
Following the above Lemma, for all $k \in \mathbb{N}$ such that $2k\geq t'$, we have
\begin{equation}\label{MWU3 1st important step}
\text{RE}(\vl^*\|{\vl}_{2k+1})-\text{RE}(\vl^*\|{\vl}_{2k-1}) \leq   -\frac{1}{2}\mu_{2k}\alpha_{2k}(v-f({\vl}_{2k-1})),
\end{equation}
where we denote $f(\vl_{2k-1})= \langle \vl_{2k-1}, \vd_{\hat{\pi}_{2k}} \rangle$. Thus, the sequence of relative entropy $\text{RE}(\vl^*\|{\vl}_{2k-1})$ is non-increasing  for all $k \geq \frac{t'}{2}$. As the sequence is bounded below by 0, it has a limit for any minimax equilibrium strategy $\vl^*$.
Since $t'$ is a finite number and $\sum_{t=1}^\infty \mu_t=\infty$, we have $\sum_{t=t'}^\infty \mu_t=\infty$. Thus, 
    \[\lim_{T\to \infty}\sum_{k=\left \lceil{\frac{t'}{2}}\right \rceil}^{T}\mu_{2k} = \infty.\]
We will prove that $\forall \epsilon >0,\; \exists h \in \mathbb{N}$ such that when the agent follows Algorithm \ref{Last round convergence in OMDPs} and the adversary follows MWU algorithm, the adversary will play strategy ${\vl}_h$ at round h and $v-f({\vl}_h) \leq \epsilon$. 
In particular, we prove this by contradiction. That is, suppose that $\exists \epsilon >0$ such that $\forall h \in \mathbb{N},\; v-f({\vl}_h) > \epsilon$. Then $\forall k \in \mathbb{N}$, 
\[
\alpha_{2k}(v-f({\vl}_{2k-1}))= \frac{(v-f({\vl}_{2k-1}))^2}{\beta} > \frac{\epsilon^2}{\beta}.\]
Let $k$ vary from $\left \lceil{\frac{t'}{2}}\right \rceil$  to T in Equation (\ref{MWU3 1st important step}). By summing over $k$, we obtain: 
\begin{equation*}
    \begin{aligned}
    \text{RE}(\vl^*\|{\vl}_{2T+1})&\leq \text{RE}(\vl^*\|{\vl}_{t'}) - \frac{1}{2}\sum_{k=\left \lceil{\frac{t'}{2}}\right \rceil}^{T} \mu_{2k}\alpha_{2k}(v-f({\vl}_{2k-1})) \\
    &\leq \text{RE}(\vl^*\|{\vl}_{t'})-\frac{1}{2}\frac{e^2}{\beta}\sum_{k=\left \lceil{\frac{t'}{2}}\right \rceil}^{T}\mu_{2k}.
    \end{aligned}
\end{equation*}
Since $\lim_{T\to \infty}\sum_{k=\left \lceil{\frac{t'}{2}}\right \rceil}^{T}\mu_{2k}= \infty$ and $\text{RE}(\vl^*\|{\vl}_{T+1})\geq 0$, it contradicts our assumption about $\forall h \in \mathbb{N},\; v-f({\vl}_h) > \epsilon$. 

Now, we take a sequence of $\epsilon_k>0$ such that $\lim_{k \to \infty}\epsilon_k=0$. Then for each k, there exists ${\vl}_{t_k}\in \Delta_n$ such that $v-\epsilon_k \leq f({\vl}_{t_k})\leq v.$ As $\Delta_n$ is a compact set and ${\vl}_{t_k}$ is bounded then following the Bolzano-Weierstrass theorem, there is a convergence subsequence ${\vl}_{\bar{t}_k}$. The limit of that sequence, ${\bar{\vl}}^*$, is a minimax equilibrium strategy of the row player (since $f({\bar{\vl}}^*)=f(\lim_{k \to \infty} {\vl}_{\bar{t}_k})=\lim_{k \to \infty}f({\vl}_{\bar{t}_k})=v$). Combining with the fact that $\text{RE}({{\bar{\vl}}^*}\|{\vl}_{2k-1})$ is non-increasing for $k\geq \left \lceil{\frac{t'}{2}}\right \rceil$ and $\text{RE}({\bar{\vl}}^*\|{\bar{\vl}}^*) =0$, we have $\lim_{k \to \infty}\text{RE}({\bar{\vl}}^*\|{\vl}_{2k-1})=0$. We also note that 
\begin{equation*}
\begin{aligned}
    \text{RE}(\bar{\vl}^*\|{\vl}_{2k})- \text{RE}(\bar{\vl}^*\|{\vl}_{2k-1})
    &=-\mu_{2k-1}
    \langle \bar{\vl}^*,\vd_{\pi_{2k-1}} \rangle+\log\left(\sum_{i=1}^n {{\vl}_{2k-1}}(i)e^{\mu_{2k-1}{\langle {\ve}_i},\vd_{\pi^*} \rangle}\right)\\
    &\leq -\mu_{2k-1}v+\log\left(\sum_{i=1}^n {{\vl}_{2k-1}}(i)e^{\mu_{2k-1}v}\right) =0,
\end{aligned}
\end{equation*}
following the fact that $\langle {\bar{\vl}^*}, \vd_{\pi} \rangle \geq v$ for all $\pi \in \Pi$ and $\langle \vl, \vd_{\pi^*} \rangle \leq v$ for all $\vl$. Thus, we have $\lim_{k \to \infty}\text{RE}({\bar{\vl}}^*\|{\vl}_{2k})=0$ as well. Subsequently, 
$\lim_{t \to \infty}\text{RE}({\bar{\vl}}^*\|{\vl}_{t}) =0$, which concludes the proof. 
\end{proof}
\end{document}